\documentclass[letter]{article}

\usepackage[nonatbib,final]{neurips_2020}

\usepackage[utf8]{inputenc} 
\usepackage[T1]{fontenc}    
\usepackage{hyperref}       
\usepackage{url}            
\usepackage{booktabs}       
\usepackage{nicefrac}       
\usepackage{microtype}      
\usepackage[fleqn]{amsmath}
\usepackage{amsfonts,amssymb,amsthm}
\usepackage[noend]{algpseudocode}
\usepackage{algorithm}
\usepackage{bm}
\usepackage{enumitem}
\usepackage{wrapfig}
\usepackage{caption}
\usepackage{subcaption}

\usepackage{tikz}
\usetikzlibrary{positioning,shapes,calc,intersections,through}
\usetikzlibrary{quotes,angles}
\tikzstyle{dot}=[draw,fill,shape=circle,inner sep=0pt,minimum size=3pt]
\tikzstyle{ps}=[circle,draw, fill=black, minimum size=4,inner sep=0pt, outer sep=0pt]
\tikzstyle{ns}=[circle,draw, fill=white, minimum size=4,inner sep=0pt, outer sep=0pt]
\tikzstyle{mve}=[circle,draw,densely dotted,thick]
\tikzstyle{ell}=[circle,draw,thick]
\usepackage{tkz-euclide}

\newcommand{\bx}{\boldsymbol{x}}

\newcommand{\lambdastar}{\lambda^{\star}}

\newcommand{\scO}{\mathcal{O}}

\DeclareMathOperator*{\trace}{tr}
\DeclareMathOperator*{\diag}{diag}

\newcommand{\field}[1]{\mathbb{#1}}
\newcommand{\R}{\field{R}}

\newcommand{\E}{\field{E}}
\renewcommand{\Pr}{\field{P}}

\newcommand{\theset}[2]{ \left\{ {#1} \,:\, {#2} \right\} }
\newcommand{\Ind}[1]{ \field{I}\left\{{#1}\right\} }
\newcommand{\norm}[1]{ \left\|{#1}\right\| }

\newcommand{\ve}{\varepsilon}
\newcommand{\C}{\mathcal{C}}

\newcommand{\Hyp}{\mathcal{H}}

\renewcommand{\hat}{\widehat}
\renewcommand{\bar}{\overline}
\renewcommand{\epsilon}{\ve}

\newtheorem{lemma}{Lemma}
\newtheorem{theorem}{Theorem}

\newtheorem{definition}{Definition}

\newcommand{\ai}{\boldsymbol{a}_i}
\newcommand{\bi}{\boldsymbol{b}_i}
\newcommand{\ci}{\boldsymbol{c}_i}

\newcommand{\tp}[1]{{#1}^{\intercal}}
\newcommand{\bp}{\boldsymbol{p}}
\newcommand{\bq}{\boldsymbol{q}}
\newcommand{\bu}{\boldsymbol{u}}
\newcommand{\bv}{\boldsymbol{v}}
\newcommand{\by}{\boldsymbol{y}}
\newcommand{\bw}{\boldsymbol{w}}
\newcommand{\bz}{\boldsymbol{z}}
\newcommand{\bh}{\boldsymbol{h}}
\newcommand{\bOne}{\boldsymbol{1}}

\newcommand{\sign}{\operatorname{sgn}}
\DeclareMathOperator{\poly}{poly}
\newcommand{\Hs}{\mathcal{H}}
\newcommand{\scq}{\textsc{scq}}
\newcommand{\cmp}{\textsc{cmp}}
\newcommand{\EL}{E}
\newcommand{\MVE}{\EL_{\mathrm{J}}}
\newcommand{\ELOUT}{\EL}
\newcommand{\ELIN}{\EL_{\mathrm{in}}}

\newcommand{\AlgoClean}{{\normalfont TessellationLearn}}
\newcommand{\AlgoFULL}{{\normalfont \textsc{recur}}}
\newcommand{\bmu}{\boldsymbol{\mu}}

\newcommand{\bc}{\boldsymbol{c}}

\newcommand{\dotp}[1]{\left\langle{#1}\right\rangle}

\newcommand{\rank}{\operatorname{rank}} 
\newcommand{\conv}{\operatorname{conv}} 

\newcommand{\ellstar}{\ell^{\star}}

\newcommand{\bmustar}{\bmu^{\star}}

\newcommand{\Mstar}{M^{\star}}
\renewcommand{\norm}[2]{\|#2\|_{#1}}
\newcommand{\orig}{\boldsymbol{0}}
\newcommand{\Bo}{\mathcal{B}_0}
\newcommand{\strtch}{{\Phi}}
\newcommand{\SCbar}{S_{\bar{C}}}

\newcommand{\ErrClust}{\triangle}
\newcommand{\SymDif}{\triangle}

\title{Exact Recovery of Mangled Clusters \\ with Same-Cluster Queries}

\author{Marco Bressan\thanks{Most of this work was done while the author was at the Sapienza University of Rome.}
\\
Dept.\ of CS, Univ.\ of Milan, Italy
\\
marco.bressan@unimi.it
\And
Nicolò Cesa-Bianchi
\\
DSRC \& Dept.\ of CS, Univ.\ of Milan, Italy
\\
nicolo.cesa-bianchi@unimi.it
\AND
Silvio Lattanzi
\\ 
Google
\\
silviol@google.com
\And
Andrea Paudice
\\
Dept.\ of CS, Univ.\ of Milan, Italy \& \\
Istituto Italiano di Tecnologia, Italy
\\
andrea.paudice@unimi.it
}

\begin{document}

\maketitle

\begin{abstract}
We study the cluster recovery problem in the semi-supervised active clustering framework. Given a finite set of input points, and an oracle revealing whether any two points lie in the same cluster, our goal is to recover all clusters exactly using as few queries as possible. To this end, we relax the spherical $k$-means cluster assumption of Ashtiani et al.\ to allow for arbitrary ellipsoidal clusters with margin. This removes the assumption that the clustering is center-based (i.e., defined through an optimization problem), and includes all those cases where spherical clusters are individually transformed by any combination of rotations, axis scalings, and point deletions. We show that, even in this much more general setting, it is still possible to recover the latent clustering exactly using a number of queries that scales only logarithmically with the number of input points. More precisely, we design an algorithm that, given $n$ points to be partitioned into $k$ clusters, uses $\scO(k^3 \ln k \ln n)$ oracle queries and $\widetilde{\scO}(kn + k^3)$ time to recover the clustering with zero misclassification error. The $\scO(\cdot)$ notation hides an exponential dependence on the dimensionality of the clusters, which we show to be necessary thus characterizing the query complexity of the problem. Our algorithm is simple, easy to implement, and can also learn the clusters using low-stretch separators, a class of ellipsoids with additional theoretical guarantees. Experiments on large synthetic datasets confirm that we can reconstruct clusterings exactly and efficiently.
\end{abstract}

\section{Introduction}
Clustering is a central problem of unsupervised learning with a wide range of applications in machine learning and data science.
The goal of clustering is to partition a set of points in different groups, so that similar points are assigned to the same group and dissimilar points are assigned to different groups.
A basic formulation is the $k$-clustering problem, in which the input points must be partitioned into $k$ disjoint subsets.
A typical example is center-based $k$-clustering, where the points lie in a metric space and one is interested in recovering $k$ clusters that minimize the distance between the points and the cluster centers.
Different variants of this problem, captured by the classic $k$-center, $k$-median, and $k$-means problems, have been extensively studied for several decades~\cite{ahmadian2019better, gonzalez1985clustering, li2016approximating}.

In this work we investigate the problem of recovering a latent clustering in the popular semi-supervised active clustering model of Ashtiani et al.~\cite{ashtiani2016clustering}.
In this model, we are given a set $X$ of $n$ input points in $\R^d$ and access to an oracle.
The oracle answers same-cluster queries (SCQs) with respect to a fixed but unknown $k$-clustering and tells whether any two given points in $X$ belong to the same cluster or not.
The goal is to design efficient algorithms that recover the latent clustering while asking as few oracle queries as possible.
Because SCQ queries are natural in crowd-sourcing systems, this model has been extensively studied both in theory~\cite{ailon2018approximate, ailon2017approximate, gamlath2018semi, huleihel2019same, mazumdar2017semisupervised, mazumdar2017clustering, NIPS2017_7054, saha2019correlation, vitale2019flattening} and in practice~\cite{firmani2018robust, gruenheid2015fault, verroios2015entity, verroios2017waldo} --- see also \cite{emamjomeh2018adaptive} for other types of queries.
In their work~\cite{ashtiani2016clustering}, Ashtiani et al.\ showed that by using $\scO(\ln n)$ same-cluster queries one can recover the optimal $k$-means clustering of $X$ in polynomial time, whereas doing so without the queries would be computationally hard.
Unfortunately,~\cite{ashtiani2016clustering} relies crucially on a strong separation assumption, called $\gamma$-margin condition: for every cluster $C$ there must exist a sphere $S_C$, centered in the centroid $\mu_C$ of $C$, such that $C$ lies entirely inside $S_C$ and every point not in $C$ is at distance $(1+\gamma)r_C$ from $\mu_C$, where $r_C$ is the radius of $S_C$.
Thus, although~\cite{ashtiani2016clustering} achieves cluster recovery with $\scO(\ln n)$ queries, it does so only for a very narrow class of clusterings.

\begin{wrapfigure}{r}{.45\linewidth}
\includegraphics[width=.45\textwidth]{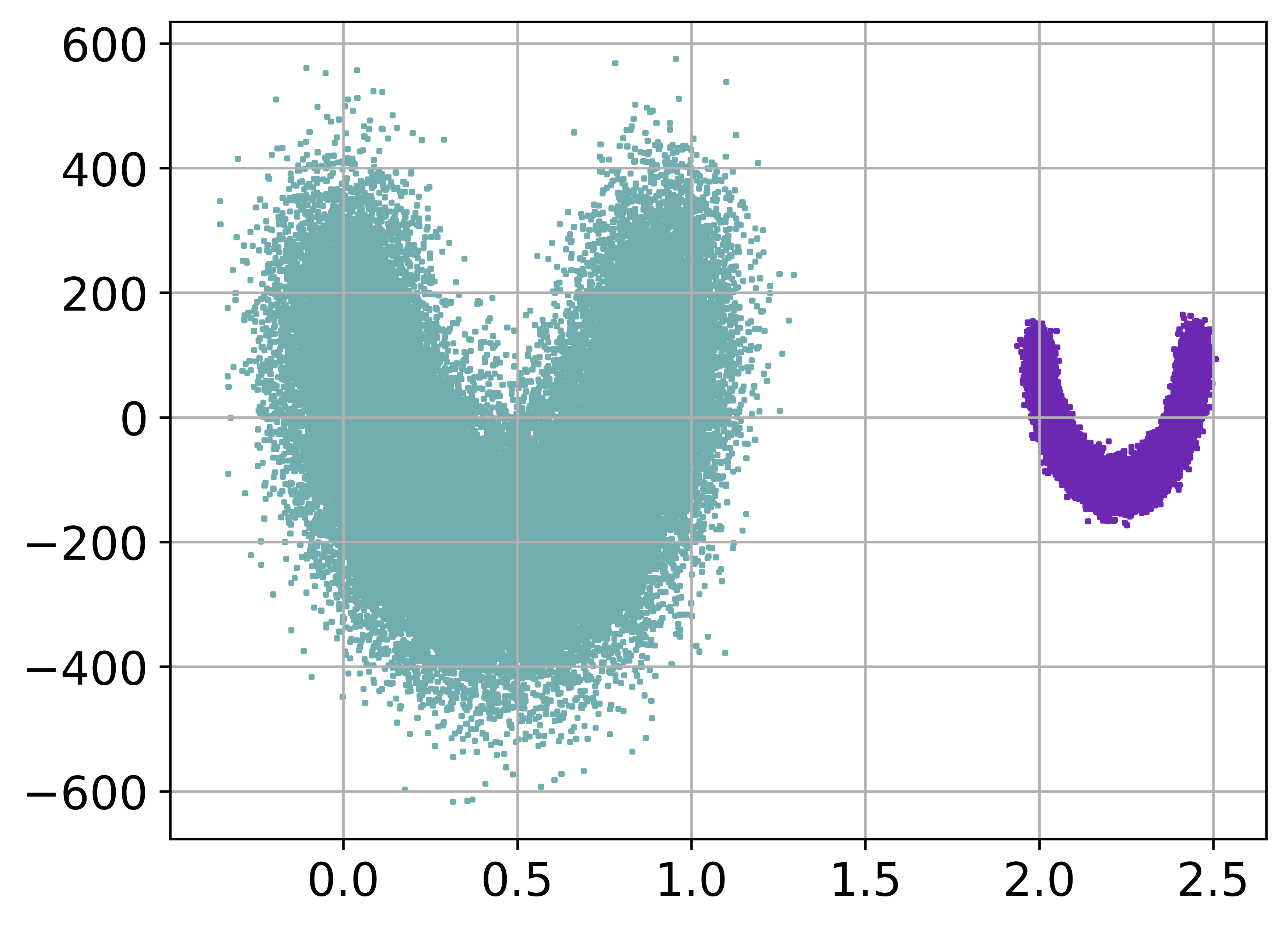}
\\
\includegraphics[width=.45\textwidth]{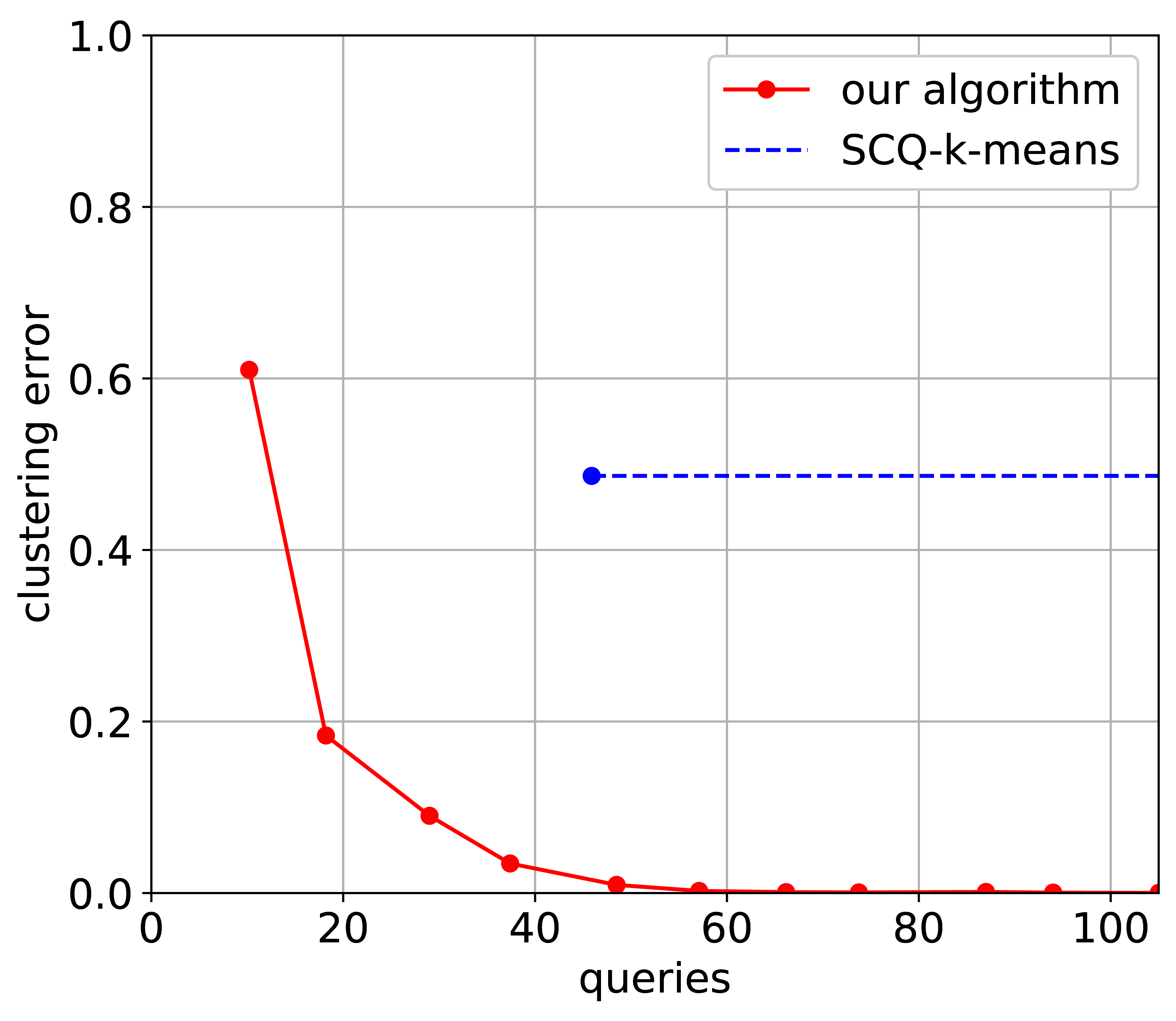}
\caption{A toy instance on $10^5$ points that we solve exactly with 105 queries, while the \scq-$k$-means algorithm of~\cite{ashtiani2016clustering} is no better than random labeling.}
\label{fig:3clust}
\end{wrapfigure}

In this work we significantly enlarge the class of clusterings that can be efficiently recovered.
We do so by relaxing the $\gamma$-margin condition of~\cite{ashtiani2016clustering} in two ways (see Section~\ref{sec:prel} for a formal definition).
First, we assume that every cluster $C$ has $\gamma$-margin in some latent space, obtained by linearly transforming all points according to some unknown positive semi-definite matrix $W_C$.
This is equivalent to assume that $C$ is bounded by an ellipsoid (possibly degenerate) rather than by a sphere (which corresponds to $W_C=I$). 
This is useful because in many real-world applications the features are on different scales, and so each cluster tends to be distorted along specific directions causing ellipsoids to fit the data better than spheres ~\cite{dzogang2012ellipsoidal, kameyama1999semiconductor, marica2014hyper, moshtaghi2011clustering, shuhua2013ellipsoids}.
Second, we allow the center of the ellipsoid to lie anywhere in space --- in the centroid of $C$ or anywhere else, even outside the convex hull of $C$.
This includes as special cases clusterings in the latent space which are solutions to $k$-medians, $k$-centers, or one of their variants.
It is not hard to see that this setting captures much more general and challenging scenarios.
For example, the latent clustering can be an optimal solution of $k$-centers where some points have been adversarially deleted and the features adversarially rescaled before the input points are handed to us.
In fact, the latent clustering need not be the solution to an optimization problem, and in particular need not be center-based: it can be literally \emph{any} clustering, as long as it respects the margin condition just described.

Our main result is that, even in this significantly more general setting, it is still possible to recover the latent clustering \emph{exactly}, in polynomial time, and using only $\scO(\ln n)$ same-cluster queries.
The price to pay for this generality is an exponential dependence of the number of queries on the dimension $d$ of the input space; this dependence is however unavoidable, as we show via rigorous lower bounds.
Our algorithm is radically different from the one in~\cite{ashtiani2016clustering}, which we call \scq-$k$-means here.
The reason is that \scq-$k$-means uses same-cluster queries to estimate the clusters' centroids and find their spherical boundaries via binary search.
Under our more general setting, however, the clusters are not separated by spheres centered in their centroids, and thus \scq-$k$-means fails, as shown in Figure~\ref{fig:3clust} (see Section~\ref{sec:exp} for more experiments).
Instead of binary search, we develop a geometric technique, based on careful tessellations of minimum-volume enclosing ellipsoids (MVEEs).
The key idea  is that MVEEs combine a low VC-dimension, which makes learning easy, with a small volume, which can be decomposed in easily classifiable elements.
While MVEEs are not guaranteed to be consistent with the cluster samples, our results can be also proven using consistent ellipsoids that are close to the convex hull of the samples. This notion of low-stretch consistent ellipsoid is new, and may be interesting in its own right.

\section{Preliminaries and definitions}
\label{sec:prel}
All missing statements and proofs can be found in the supplementary material.
The input to our problem is a triple $(X,k,\gamma)$ where $X \subset \R^d$ is a set of $n$ arbitrary points, $k \ge 2$ is an integer, and $\gamma \in \R_{>0}$ is the margin (see below).
We assume there exists a latent clustering $\C = \{C_1,\ldots,C_k\}$ over the input set $X$, which we do not know and want to compute.
To this end, we are given access to an oracle answering \textsl{same-cluster queries}: a query $\scq(\bx,\bx')$ is answered by $+1$ if $\bx,\bx'$ are in the same cluster of $\C$, and by $-1$ otherwise.
Our goal is to recover $\C$ while using as few queries as possible.
Note that, given any subset $S \subseteq X$, with at most $k |S|$ queries one can always learn the label (cluster) of each $\bx \in S$ up to a relabeling of $\C$, see~\cite{ashtiani2016clustering}.

It is immediate to see that if $\C$ is completely arbitrary, then no algorithm can reconstruct $\C$ with less than $n$ queries.
Here, we assume some structure by requiring each cluster to satisfy a certain \emph{margin} condition, as follows.
Let $W \in \R^{d \times d}$ be some positive semidefinite matrix (possibly different for each cluster).
Then $W$ induces the seminorm
$
    \norm{W}{\bx} = \sqrt{\bx^{\top}W\bx}
$, which in turn induces the pseudo-metric
$
    d_W(\bx,\by) = \norm{W}{\bx-\by}
$.
The same notation applies to any other PSD matrix, and when the matrix is clear from the context, we drop the subscript and write simply $d(\cdot,\cdot)$.
The margin condition that we assume is the following:
\begin{definition}[Clustering margin]
\label{def:smp}
A cluster $C$ has margin $\gamma > 0$ if there exist a PSD matrix $W=W(C)$ and a point $\bc \in \R^d$ such that for all $\by \notin C$ and all $\bx \in C$ we have $d_W(\by,\bc) > \sqrt{1+\gamma} \, d_W(\bx,\bc)$.
If this holds for all clusters, then we say that the clustering $\C$ has margin $\gamma$.
\end{definition}
This is our only assumption.
In particular, we do not assume the cluster sizes are balanced, or that $\C$ is the solution to an optimization problem, or that points in a cluster $C$ are closer to the center of $C$ than to the centers of other clusters.
Note that the matrices $W$ and the points $\bc$ are unknown to us.
The spherical $k$-means setting of~\cite{ashtiani2016clustering} corresponds to the special case where for every $C$ we have $W= r I$ for  some $r=r(C) > 0$ and $\bc = \bmu(C) = \frac{1}{|C|}\sum_{\bx \in C} \bx$.

We denote a clustering returned by our algorithm by $\hat{\C}=\{\hat{C}_1,\ldots,\hat{C}_k\}$.
The quality of $\hat{\C}$ is measured by the disagreement with $\C$ under the best possible relabeling of the clusters, that is, $\ErrClust(\hat{\C},\C)=\min_{\sigma \in S_k} \frac{1}{2n}\sum_{i=1}^k |C_1 \SymDif \hat{C}_{\sigma(i)}|$, where $S_k$ is the set of all permutations of $[k]$.
Our goal is to minimize $\ErrClust(\hat{\C},\C)$ using as few queries as possible. In particular, we characterize the query complexity of exact reconstruction, corresponding to $\ErrClust(\hat{\C},\C)=0$.
The \emph{rank} of a cluster $C$, denoted by $\rank(C)$, is the rank of the subspace spanned by its points.

\section{Our contribution}
Our main contribution is an efficient active clustering algorithm, named \AlgoFULL, to recover the latent clustering exactly.
We show the following.
\begin{theorem}
\label{th:main}
Consider any instance $(X,k,\gamma)$ whose latent clustering $\C$ has margin $\gamma$.
Let $n = |X|$, let $r \le d$ be the maximum rank of a cluster in $\C$, and let
$f(r,\gamma) = \max\big\{2^r,\scO\big(\frac{r}{\gamma}\ln\!\frac{r}{\gamma}\big)^r\big\}$.
Given $(X,k,\gamma)$, \AlgoFULL\ with probability $1$ outputs $\C$ (up to a relabeling), and with high probability runs in time $\scO((k \ln n)(n+k^2 \ln k))$ using $\scO\big((k \ln n) \, (k^2 d^2 \ln k + f(r,\gamma))\big)$ same-cluster queries.
\end{theorem}
More in general, \AlgoFULL\ clusters correctly $(1-\epsilon)n$ points using $\scO\big((k \ln \nicefrac{1}{\epsilon}) \, (k^2 d^2 \ln k + f(r,\gamma))\big)$ queries in expectation.
Note that the query complexity depends on $r$ rather than on $d$, which is desirable as real-world data often exhibits a low rank (i.e., every point can be expressed as a linear combination of at most $r$ other points in the same cluster, for some $r\!\ll\!d$).
In addition, unlike the algorithm of~\cite{ashtiani2016clustering}, which is Monte Carlo and thus can fail, \AlgoFULL\ is Las Vegas: it returns the correct clustering with probability $1$, and the randomness is only over the number of queries and the running time.
Moreover, \AlgoFULL\ is simple to understand and easy to implement.
It works by recovering a constant fraction of some cluster at each round, as follows (see Section~\ref{sec:single} and Section~\ref{sec:all}):
\begin{enumerate}[leftmargin=*,topsep=0pt,parsep=0pt,itemsep=2pt]
\item \textbf{Sampling.}
We draw points uniformly at random from $X$ until, for some cluster $C$, we obtain a sample $S_C$ of size $\simeq d^2$.
We can show that with good probability $|C| \simeq \frac{1}{k}\,|X|$, and that, by standard PAC bounds, any ellipsoid $E$ containing $S_C$ contains at least half of $C$.
\item \textbf{Computing the MVEE.}
We compute the MVEE (minimum-volume enclosing ellipsoid) $\EL=\MVE(S_C)$ of $S_C$.
As said, by PAC bounds, $\EL$ contains at least half of $C$.
If we were lucky, $\EL$ would not contain any point from other clusters, and $\EL \cap X$ would be our large subset of $C$.
Unfortunately, $\EL$ can contain an arbitrarily large number of points from $X \setminus C$.
Our goal is to find them and recover $C \cap \EL$.
\item \textbf{Tessellating the MVEE.}
To recover $C \cap E$, we partition $E$ into roughly $(\nicefrac{d}{\gamma})^d$ hyperrectangles, each one with the property of being monochromatic: its points are either all in $C$ or all in $X \setminus C$.
Thanks to this special tessellation, with roughly $(\nicefrac{d}{\gamma})^d$ queries we can find all hyperrectangles containing only points of $C$, and thus compute $C \cap \EL$.
\end{enumerate}

Our second contribution is to show a family of instances where every algorithm needs roughly $(\nicefrac{1}{\gamma})^r$ same-cluster queries to return the correct clustering.
This holds even if the algorithm is allowed to fail with constant probability.
Together with Theorem~\ref{th:main}, this gives an approximate characterization of the query complexity of the problem as a function of $\gamma$ and $r$.
That is, for ellipsoidal clusters, a margin of $\gamma$ is necessary and sufficient to achieve a query complexity that grows roughly as $(\nicefrac{1}{\gamma})^r$.
This lower bound also implies that our algorithm is nearly optimal, even compared to algorithms that can fail.
The result is given formally in Section~\ref{sec:lb}.

Our final contribution is a set of experiments on large synthetic datasets. They show that our algorithm \AlgoFULL\ achieves exact cluster reconstruction efficiently, see Section~\ref{sec:exp}.

\section{Related work.}
The semi-supervised active clustering (SSAC) framework was introduced in~\cite{ashtiani2016clustering}, together with the \scq-$k$-means algorithm that recovers $\C$ using $O(k^2 \ln k + k \ln n)$ same-cluster queries.
This is achieved via binary search under assumptions much stronger than ours (see above).
In our setting, \scq-$k$-means works only when every point $\bc$ is close to the cluster centroid and the condition number of $W$ is small (see the supplementary material); indeed, our experiments show that \scq-$k$-means fails even when $W \simeq I$.
Interestingly, even if binary search and its generalizations are at the core of many active learning techniques~\cite{nowak2011geometry}, here they do not seem to help. 
We remark that we depend on $\gamma$ in the same way as~\cite{ashtiani2016clustering}: if $\gamma$ is a lower bound on the actual margin of $\C$, then the correctness is guaranteed, otherwise we may return any clustering.
Clustering with same-cluster queries is also studied in~\cite{mazumdar2017semisupervised}, but they assume stochastic similarities between points that do not necessarily define a metric. Same-cluster queries for center-based clustering in metric spaces were also considered by~\cite{sanyal2019semi}, under $\alpha$-center proximity \cite{awasthi2012center} instead of $\gamma$-margin (see \cite[Appendix B]{ashtiani2016clustering} for a comparison between the two notions).
Finally,~\cite{ailon2017approximate} used same-cluster queries to obtain a PTAS for $k$-means.
Unfortunately, this gives no guarantee on the clustering error: a good $k$-means value can be achieved by a clustering very different from the optimal one, and vice versa.
From a more theoretical viewpoint, the problem has been extensively studied for clusters generated by a latent mixture of Gaussians~\cite{dasgupta1999learning, kalai2010efficiently, hardt2015tight}.

As same-cluster queries can be used to label the points, one can also learn the clusters using standard pool-based active learning tools. For example, using quadratic feature expansion, our ellipsoidal clusters can be learned as hyperplanes. Unfortunately, the worst-case label complexity of actively learning hyperplanes with margin $\gamma < \nicefrac{1}{2}$ is still $\Omega\big((R/\gamma)^d\big)$, where $R$ is the radius of the smallest ball enclosing the points~\cite{Gonen&2013}.
Some approaches that bypass this lower bound have been proposed. In \cite{Gonen&2013} they prove an approximation result, showing that $\mathrm{OPT}\times\scO\big(d\ln\frac{R}{\gamma}\big)$ queries are sufficient to learn any hyperplane with margin $\gamma$, where $\mathrm{OPT}$ is the number of queries made by the optimal active learning algorithm. Moreover, under distributional assumptions, linear separators can be learned efficiently with roughly $\scO(d \ln n)$ label queries \cite{balcan2007margin,balcan2013active,dasgupta2005analysis}. In a different line of work, \cite{Kane19} show that $\scO\big((d\ln n)\ln\frac{R}{\gamma}\big)$ queries suffice for linear separators with margin $\gamma$ when the algorithm can also make comparison queries: for any two pairs of points $(\bx,\bx')$ and $(\by,\by')$ from $X$, a comparison query returns $1$ iff $d_W(\bx,\bx') \le d_W(\by,\by')$.
As we show, comparison queries do not help learning the latent metric $d_W$ using metric learning techniques \cite{kulis2013metric} (see the supplementary material).
In general, the query complexity of pool-based active learning is characterized by the star dimension of the family of sets \cite{hanneke2015minimax}.
This implies that, if we allow for a non-zero probability of failure, then $\scO(\mathfrak{s}\ln n)$ queries are sufficient for reconstructing a single cluster, where $\mathfrak{s}$ is the star dimension of ellipsoids with margin $\gamma$.
To the best of our knowledge, this quantity is not known for ellipsoids with margin (not even for halfspaces with margin), and our results seem to suggest a value of order $(\nicefrac{d}{\gamma})^d$.
If true, this would imply then the general algorithms of \cite{hanneke2015minimax} could be used to solve our problem with a number of queries comparable to ours. 
However, note that our reconstructions are exact with probability one, and are achieved by simple algorithms that work well in practice.

\section{Recovery of a single cluster with one-sided error}
\label{sec:meat}
\label{sec:single}
This section describes the core of our cluster recovery algorithm.
The main idea is to show that, given any subset $S_C \subseteq C$ of some cluster $C$, if we compute a small ellipsoid $\EL$ containing $S_C$, then we can compute $C \cap \EL$ deterministically with a small number of queries.

Consider a subset $S_C \subseteq C$, and let $\conv(S_C)$ be its convex hull.
The \emph{minimum-volume enclosing ellipsoid} (MVEE) of $S_C$, also known as L\"owner-John ellipsoid and denoted by $\MVE(S_C)$, is the volume-minimizing ellipsoid $\EL$ such that $S_C \subset \EL$ (see, e.g., \cite{todd2016minimum}).
The main result of this section is that $C \cap \MVE(S_C)$ is easy to learn. Formally, we prove:
\begin{theorem}
\label{thm:single}
Suppose we are given a subset $S_C \subseteq C$, where $C$ is any unknown cluster.
Then we can learn $C \cap \MVE(S_C)$ using  $\max\!\big\{2^r,\scO\big(\frac{r}{\gamma}\ln\!\frac{r}{\gamma}\big)^{r}\big\}$ same-cluster queries, where $r=\rank(C)$ and $\MVE(S_C)$ is the minimum-volume enclosing ellipsoid of $S_C$.
\end{theorem}
In the rest of the section we show how to learn $C \cap \MVE(S_C)$ and sketch the proof of the theorem.

\paragraph{The MVEE.}
The first idea is to compute an ellipsoid $\EL$ that is ``close'' to $\conv(S_C)$.
A $d$-rounding of $S_C$ is any ellipsoid $\EL$ satisfying the following (we assume the center of $E$ is the origin):
\begin{align}
\frac{1}{d}\EL \subseteq \conv(S_C) \subseteq \EL
\label{eq:LJ}
\end{align}
In particular, by a classical theorem by John~\cite{KhachiyanMVE}, the MVEE $\MVE(S_C)$ is a $d$-rounding of $S_C$.
We therefore let $\EL=\MVE(S_C)$.
Note however that \emph{any} $d$-rounding ellipsoid $\EL$ can be chosen instead, as the only property we exploit in our proofs is~\eqref{eq:LJ}.
%

It should be noted that the ambient space dimensionality $d$ can be replaced by $r=\rank(S_C)$.
To this end, before computing $\EL=\MVE(S_C)$, we compute the span $V$ of $S_C$ and a canonical basis for it using a standard algorithm (e.g., Gram-Schmidt).
We then use $V$ as new ambient space, and search for $\MVE(S_C)$ in $V$.
This works since  $\MVE(S_C) \subset V$, and lowers the dimensionality from $d$ to $r \le d$.
From this point onward we still use $d$ in our notation, but all our constructions and claims hold unchanged if instead one uses $r$, coherently with the bounds of Theorem~\ref{thm:single}.

\paragraph{The monochromatic tessellation.}
We now show that, by exploiting the $\gamma$-margin condition, we can learn $C \cap \MVE(S_C)$ with a small number of queries.
We do so by discretizing $\MVE(S_C)$ into hyperrectangles so that, for each hyperrectangle, we need only one query to decide if it lies in $C$ or not.
The crux is to show that there exists such a discretization, which we call \emph{monochromatic tessellation}, consisting of relatively few hyperrectangles, roughly $(\frac{d}{\gamma}\ln\!\frac{d }{\gamma})^{d}$.

\setlength{\intextsep}{5pt}
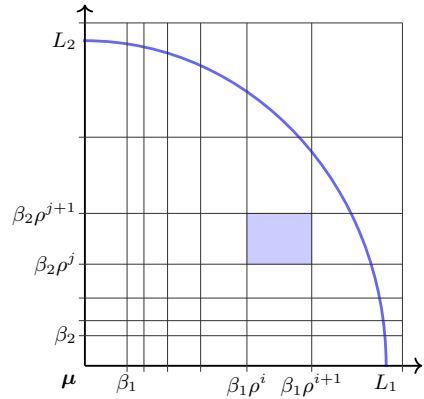
\begin{wrapfigure}{r}{.39\linewidth}
    \centering
    \scalebox{.89}{\begin{tikzpicture}[scale=1.8]
\tikzstyle{grid}=[black!80!white,thin];
\pgfmathsetmacro\betax{0.35}; 
\pgfmathsetmacro\betay{0.25};
\pgfmathsetmacro\rx{1.4}; 
\pgfmathsetmacro\ry{1.5};
\pgfmathsetmacro\b{6};
\pgfmathsetmacro\xlim{\betax*(\rx)^\b};
\pgfmathsetmacro\ylim{\betay*(\ry)^\b};
\pgfmathsetmacro\Lx{2.5};
\pgfmathsetmacro\Ly{2.7};
\pgfmathsetmacro\xi{\betax*(\rx)^4};
\pgfmathsetmacro\xii{\betax*(\rx)^5};
\pgfmathsetmacro\yi{\betay*(\ry)^3};
\pgfmathsetmacro\yii{\betay*(\ry)^4};

\clip(-.6,-.27) rectangle (2.81,3.1);

\draw[black,thick,->] (0,0) -- (2.8,0);
\draw[black,thick,->] (0,0) -- (0,3);

\begin{scope}
\clip(-.008,-.008) rectangle (5.6,2.9);
\draw[very thick,blue,opacity=.6] (0,0) circle [x radius=\Lx, y radius=\Ly]; 
\end{scope}

\foreach \i in {0,...,\b}{
    \pgfmathsetmacro\x{\betax*(\rx)^\i};
    \draw[grid] (\x,-0.05) -- (\x,\ylim);
    \pgfmathsetmacro\y{\betay*(\ry)^\i};
    \draw[grid] (-0.05,\y) -- (\xlim,\y);
}
\node[below left] (orig) at (0,0) {\small $\bmu$};
\node[below] (beta1) at (\betax,0) {\small $\beta_1$};
\node[left] (beta2) at (0,\betay) {\small $\beta_2$};


\node[below] (beta1i) at (\xi,0) {\small $\beta_1 \rho^i$};
\node[below] (beta1i1) at (\xii,0) {\small $\beta_1 \rho^{i+1}$};
\node[left] (beta2i) at (0,\yi) {\small $\beta_2 \rho^j$};
\node[left] (beta2i1) at (0,\yii) {\small $\beta_2 \rho^{j+1}$};
\node[below] (Lx) at (\Lx,0) {\small $L_1$};
\node[left] (Ly) at (0,\Ly) {\small $L_2$};

\draw[fill=blue,opacity=.2] (\xi,\yi) rectangle (\xii,\yii);
\end{tikzpicture}}
    \caption{\small The tessellation $\mathcal{R}$ of $\ELOUT \cap \R^d_+$.
    Every hyperrectangle $R$ (shaded) is such that $R \cap \ELOUT$ is monochromatic, i.e.\ contains only points of $C$ or of $X \setminus C$.}
    \label{fig:tess}
\end{wrapfigure}
Let $\ELOUT = \MVE(S_C)$. To describe the monochromatic tessellation, we first define the notion of monochromatic subset:
\begin{definition}
A set $B \subset \R^d$ is \emph{monochromatic} with respect to a cluster $C$ if it does not contain two points $\bx,\by$ with $\bx \in C$ and $\by \notin C$.
\end{definition}
Fix a hyperrectangle $R \subset \R^d$. The above definition implies that, if $B=R\cap \ELOUT$ is monochromatic, then
we learn the label of all points in $B$ with a single query.
Indeed, if we take any $\by \in B$ and any $\bx \in S_C$, the query $\scq(\by,\bx)$ tells us whether $\by \in C$ or $\by \notin C$ simultaneously for all $\by \in B$.
Therefore, if we can cover $\ELOUT$ with $m$ monochromatic hyperrectangles, then we can learn $C \cap \ELOUT$ with $m$ queries.
Our goal is to show that we can do so with $m \simeq (\frac{d}{\gamma}\ln\!\frac{d}{\gamma})^{d}$.

We now describe the construction in more detail; see also Figure~\ref{fig:tess}.
The first observation is that, if any two points $\bx,\by \in X$ are such that $\bx \in C$ and $\by \notin C$, then $|x_i-y_i| \gtrsim \nicefrac{\gamma}{d}$ for some $i$.
Indeed, if this was not the case then $\bx,\by$ would be too close and would violate the $\gamma$-margin condition.
This implies that, for $\rho \simeq 1+\nicefrac{\gamma}{d}$, any hyperrectangle whose sides have the form $[\beta_i, \beta_i \rho\,]$ is monochromatic.
We can exploit this observation to construct the tessellation.
Let the semiaxes of $\ELOUT$ be the canonical basis for $\R^d$ and its center $\bmu$ be the origin.
For simplicity, we only consider the positive orthant, the argument being identical for every other orthant.
Let $L_i$ be the length of the $i$-th semiaxis of $\ELOUT$.
The goal is to cover the interval $[0,L_i]$ along the $i$-th semiaxis of $\ELOUT$ with roughly $\log_{\rho}(L_i/\beta_i)$ intervals of length increasing geometrically with $\rho$.
More precisely, we let
$
T_i = \big\{
\big[0, \beta_i \big],
\big(\beta_i, \beta_i \rho\big],
\ldots,
\big( \beta_i \rho^{b-1}, \beta_i \rho^{b} \big]
\big\}
$,
where $\beta_i>0$, $\rho>1$, and $b \ge 0$ are functions of $\gamma$ and $d$.
Then our tessellation is the cartesian product of all the $T_i$:
\begin{definition}
\label{def:R}
Let $\R^d_+$ be the positive orthant of $\R^d$.
The tessellation $\mathcal{R}$ of $\ELOUT \cap \R^d_+$ is the set of $(b+1)^d$ hyperrectangles expressed in the canonical basis $\{\bu_1,\ldots,\bu_d\}$ of $\ELOUT$:
$
\mathcal{R} = T_1  \times \ldots \times T_{d}
$.
\end{definition}
We now come to the central fact.
Loosely speaking, if $\beta_i \simeq \frac{\gamma}{d}L_i$ then the point $(\beta_1,\ldots,\beta_d)$ lies ``well inside'' $\conv(S_C)$, because~\eqref{eq:LJ} tells us $\EL$ itself is close to $\conv(S_C)$.
By setting $\rho,b$ adequately, then, we can guarantee the intervals of $T_i$ of the form $(\beta_i \rho^{j-1}, \beta_i \rho^{j}]$ cover all the space between $\conv(S_C)$ and $\EL$.
More formally we show that, for a suitable choice of $\beta_i,\rho,b$, the tessellation $\mathcal{R}$ satisfies the following properties (see the supplementary material): 
\begin{enumerate}[label={(\arabic*)},topsep=0pt,parsep=0pt,itemsep=0pt]
\item $|\mathcal{R}| \le \max\big\{1,\scO\big(\frac{d}{\gamma}\ln\!\frac{d}{\gamma}\big)^d\big\}$
\item ${\displaystyle E \cap \R_+^d \subseteq \bigcup_{R \in \mathcal{R}}R}$
\item For every $R \in \mathcal{R}$, the set $R \cap \ELOUT$ is monochromatic w.r.t.\ $C$
\end{enumerate}
Once the three properties are established, Theorem~\ref{thm:single} immediately derives from the discussion above.
\paragraph{Pseudocode.}
We list below our algorithm that learns $C \cap \ELOUT$ subject to the bounds of Theorem~\ref{thm:single}.
We start by computing $\EL=\MVE(S_C)$ and selecting the subset $\EL_X = X \cap \EL$.
We then proceed with the tessellation, but without constructing $\mathcal{R}$ explicitly.
Note indeed that, for every $\by \in \EL_X$, the hyperrectangle $R(\by)$ containing $\by$ is determined uniquely by $|y_i|/\beta_i$ for all $i \in [d]$.
In fact, we can manage all orthants at once by simply looking at $y_i/\beta_i$.
After grouping all points $\by$ by their $R(\by)$, we repeatedly take a yet-unlabeled $R$ and label it as $C$ or not $C$.
Finally, we return all points in the hyperrectangles labeled as $C$.
\begin{algorithm}[h!]
\caption{\AlgoClean($X,S_C,\gamma$)}
\begin{algorithmic}[1]
\State compute $\ELOUT \leftarrow \MVE(S_C)$ or any other $r$-rounding of $S_C$ \label{line:A1mve}
\State compute $\ELOUT_X \leftarrow X \cap \ELOUT$
\State compute $\beta_i,\rho,b$ as a function of $r,\gamma$ \Comment{see Figure~\ref{fig:tess}}
\For{every $\by \in \ELOUT_X$} \label{line:A1index}
\State map $\by$ to $R(\by)$
\EndFor
\State $\bx_C \leftarrow$ any point in $S_C$
\While{there is some unlabeled $R$} \label{line:A1loop} 
\State label$(R)\leftarrow \scq(\bx_C,\by)$, where $\by$ is any point s.t.\ $R(\by)=R$ \label{line:A1loop2}
\EndWhile
\State \Return all $\by$ mapped to $R$ such that label$(R)=+1$
\end{algorithmic}
\end{algorithm}
\paragraph{Low-stretch separators.}
We conclude this section with a technical note.
Although MVEEs enable exact cluster reconstruction, they do not give PAC guarantees since they do not ensure consistency.
Indeed, if we draw a sample $S$ from $X$ and let $S_C=S \cap C$, there is no guarantee that $\EL=\MVE(S_C)$ separates $S_C$ from $S \setminus S_C$.
On the other hand, any ellipsoid $\EL$ separating $S_C$ from $S\setminus S_C$ is a good classifier in the PAC sense, but there is no guarantee it will be close to $\conv(S_C)$, thus breaking down our algorithm.
Interestingly, in the supplementary material we show that it is possible to compute an ellipsoid that is simultaneously a good PAC classifier \emph{and} close to $\conv(S_C)$, yielding essentially the same bounds as Theorem~\ref{thm:single}.
Formally, we have:
\begin{definition}
\label{def:lss}
Given any finite set $X$ in $\R^d$ and a subset $S \subset X$, a $\strtch$-stretch separator for $S$ is any ellipsoid $\ELOUT$ separating $S$ from $X \setminus S$ and such that $\ELOUT \subseteq \strtch\MVE(S)$.
\end{definition}
\begin{theorem}
\label{thm:stretch}
Suppose $C$ has margin $\gamma > 0$ w.r.t.\ to some $\bz \in \R^d$ and fix any subset $S_C \subseteq C$.
There exists a $\strtch$-stretch separator for $S_C$ with $\strtch=64\sqrt{2}d^2\max\big\{125,\nicefrac{1}{\gamma^3}\big\}$.
\end{theorem}
%

\section{Exact recovery of all clusters}
\label{sec:all}
In this section we conclude the construction of our algorithm \AlgoFULL\ (listed below), and we bound its query complexity and running time.
\AlgoFULL\ proceeds in rounds.
At each round, it draws samples uniformly at random from $X$ until, for some sufficiently large $b>0$, it obtains a sample $S_C$ of size $b\,\! d^2 \ln k$ from some cluster $C$.
At this point, by concentration and PAC bounds, we know that any ellipsoid $\EL$ containing $S_C$ satisfies $|C \cap \EL| \ge \frac{1}{4k} |X|$ with probability at least $\nicefrac{1}{2}$.
\AlgoFULL\ uses the routine \AlgoClean$()$ from Section~\ref{sec:single} to compute such a subset $C \cap \EL$ efficiently (see Theorem~\ref{thm:single}).
\AlgoFULL\ then deletes $C \cap \EL$ from $X$ and repeats the process on the remaining points.
This continues until a fraction $(1-\epsilon)$ of points have been clustered.
In particular, when $\epsilon < \nicefrac{1}{n}$, \AlgoFULL\ clusters all the points of $X$.

\begin{algorithm}[h!]
\caption{
\label{alg:algofull}
\AlgoFULL($X,k,\gamma,\epsilon$)}
\begin{algorithmic}[1]
\State $\hat{C}_1,\ldots,\hat{C}_k \leftarrow \emptyset$
\While{$|X| > \epsilon n$}
\State draw samples with replacement from $X$ until $|S_C| \ge b d^2 \!\ln k$ for some $C$ \label{line:A2sample}
\State $C_E \leftarrow$ \AlgoClean$(X,S_C,\gamma)$ \label{line:A2clean}
\State add $C_E$ to the corresponding $\hat{C}_i$
\State $X \leftarrow X \setminus C_E$
\EndWhile
\State \Return $\hat{\C}=\{\hat{C}_1,\ldots,\hat{C}_k \}$
\end{algorithmic}
\end{algorithm}
Regarding the correctness of \AlgoFULL, we have:
\begin{lemma}
\label{lem:correctness}
The clustering $\hat{\C}$ returned by \AlgoFULL$(X,k,\gamma,\epsilon)$ deterministically satisfies $\ErrClust(\hat{\C},\C) \le \epsilon$.
In particular, for $\epsilon < \nicefrac{1}{n}$ we have $\ErrClust(\hat{\C},\C) = 0$.
\end{lemma}
This holds because $\ErrClust(\hat{\C},\C)$ is bounded by the fraction of points that are still in $X$ when \AlgoFULL\ returns; and this fraction is at most $\epsilon$ by construction.
Regarding the cost of \AlgoFULL, we have:
\begin{lemma}
\label{lem:cost}
\AlgoFULL$(X,k,\gamma,\epsilon)$ makes $\scO(k^3 \ln k \ln(\nicefrac{1}{\epsilon}))$ same-cluster queries in expectation, and for all fixed $a \ge 1$, \AlgoFULL$(X,k,\gamma,0)$ with probability at least $1-n^{-a}$ makes $\scO(k^3 \ln k \ln n)$ same-cluster queries and runs in time $\scO((k \ln n)(n + k^2 \ln k )) = \widetilde{\scO}(kn + k^3)$.
\end{lemma}
%
In the rest of the section we sketch the proof of Lemma~\ref{lem:cost}.
We start by bounding the number of rounds performed by \AlgoFULL.
Recall that, at each round, with probability at least $\nicefrac{1}{2}$ a fraction at least $\nicefrac{1}{4k}$ of points are labeled and removed.
Thus, at each round, the size of $X$ drops by $(1-\nicefrac{1}{8k})$ in expectation.
Hence, we need roughly $8 k \ln (\nicefrac{1}{\epsilon})$ rounds before the size of $X$ drops below $\epsilon n$.
Indeed, we prove:
\begin{lemma}
\label{lem:hp_rounds}
\AlgoFULL($X,k,\gamma,\epsilon$) makes at most $8 k \ln(\nicefrac{1}{\epsilon})$ rounds in expectation, and for all fixed $a \ge 1$, \AlgoFULL($X,k,\gamma,0$) with probability at least $1-n^{-a}$ performs at most $(8 k + 6 a \sqrt{k}) \ln n$ rounds.
\end{lemma}
We can now bound the query cost and running time of \AlgoFULL, by counting the work done at each round and using Lemma~\ref{lem:hp_rounds}.
To simplify the discussion we treat $d,r,\gamma$ as constants, but fine-grained bounds can be derived immediately from the discussion itself.

\textbf{Query cost of \AlgoFULL.}
The algorithm makes queries at line~\ref{line:A2sample} and line~\ref{line:A2clean}.
At line~\ref{line:A2sample}, \AlgoFULL\ draws at most $b k d^2 \ln k = \scO(k \ln k)$ samples.
This holds since there are at most $k$ clusters, so after $b k d^2 \ln k$ samples, the condition $|S_C| \ge  b d^2 \ln k$ will hold for some $C$.
Since learning the label of each sample requires at most $k$ queries, line~\ref{line:A2sample} makes $\scO(k^2 \ln k)$ queries in total.
At line~\ref{line:A2clean}, \AlgoFULL\ makes $f(d,\gamma)=\scO(1)$ queries by Theorem~\ref{thm:single}.
Together with Lemma~\ref{lem:hp_rounds}, this implies that \AlgoFULL\ with probability at least $1-n^{-a}$ makes at most $\scO(k \ln n)\times\scO(k^2 \ln k) = \scO(k^3 \ln k \ln n)$ queries.

\textbf{Running time of \AlgoFULL.}
Line~\ref{line:A2sample} takes time $\scO(k^2 \ln k)$, see above.
The rest of each round is dominated by the invocation of \AlgoClean\ at line~\ref{line:A2clean}.
Recall then the pseudocode of \AlgoClean\ from Section~\ref{sec:single}.
At line~\ref{line:A1mve}, computing $\EL=\MVE(S_C)$ or any $r$-rounding of $S_C$ takes time $\scO(|S_C|^{3.5} \ln |S_C|)$, see~\cite{KhachiyanMVE}.\footnote{More precisely, for a set $S$ an ellipsoid $\EL$ such that $\frac{1}{(1+\epsilon)d}\EL \subset \conv(S) \subset \EL$ can be computed in $\scO(|S|^{3.5} \ln(|S|/\epsilon))$ operations in the real number model of computation, see~\cite{KhachiyanMVE}.}
This is in $\widetilde{\scO}(1)$ since by construction $|S_C| = \scO(d^2 \ln k) = \widetilde{\scO}(1)$.
Computing $\EL_X = X \cap \EL$ takes time $\scO(|X| \poly(d))=\scO(n)$.
For the index (line~\ref{line:A1index}), we can build in time $\scO(|X \cap \EL|)$ a dictionary that maps every $R \in \mathcal{R}$ to the set $R \cap \EL_X$.
The classification part (line~\ref{line:A1loop}) takes time $|\mathcal{R}|=\scO(1)$.
Finally, enumerating all positive $R$ and concatenating the list of their points takes again time $\scO(|X \cap \EL| \poly(d))$.
By the rounds bound of Lemma~\ref{lem:hp_rounds}, then, \AlgoFULL\ with probability at least $1-n^{-a}$ runs in time $\scO((k \ln n)(n + k^2 \ln k ))$.

\section{Lower bounds}
\label{sec:lb}
%
We show that any algorithm achieving exact cluster reconstruction must, in the worst case, perform a number of same-cluster queries that is exponential in $d$ (the well-known ``curse of dimensionality'').
Formally, in the supplementary material we prove: 
\begin{theorem}
\label{thm:LB}
Choose any possibly randomized learning algorithm.
There exist:
\begin{enumerate}[topsep=0pt,parsep=0pt,itemsep=0pt]
    \item for all $\gamma \in (0,\nicefrac{1}{7})$ and $d \ge 2$, an instance on $n = \Omega\big((\frac{1+\gamma}{8\gamma})^{\frac{d-1}{2}} \big)$ points and $3$ clusters
    \item for all $\gamma > 0$ and $d \ge 48(1+\gamma)^2$, an instance on $n = \Omega\big( e^{\frac{d}{48(1+\gamma)^2}}\big)$ points and $2$ clusters
\end{enumerate}
such that (i) the latent clustering $\C$ has margin $\gamma$, and (ii) to return with probability $\nicefrac{2}{3}$ a $\hat{\C}$ such that $\ErrClust(\hat{\C},\C)=0$, the algorithm must make $\Omega(n)$ same-cluster queries in expectation.
\end{theorem}
The lower bound uses two different constructions, each one giving a specific instance distribution where any algorithm must perform $\Omega(n)$ queries in expectation, where $n$ is exponential in $d$ as in the statement of the theorem.
The first construction is similar to the one shown in~\cite{Gonen&2013}.
The input set $X$ is a packing of $\simeq(\nicefrac{1}{\gamma})^d$ points on the $d$-dimensional sphere, at distance $\simeq \sqrt{\gamma}$ from each other.
We show that, for $\bx=  (x_1,\ldots,x_d) \in X$ drawn uniformly at random, setting $W=(1+\gamma)\diag(x_1^2,\ldots,x_d^2)$ makes $\bx$ an outlier.
That is, $X \setminus \{\bx\}$ forms a first cluster $C_1$, and $\{\bx\}$ forms a second cluster $C_2$, and both clusters satisfy the margin condition.
In order to output the correct clustering, any algorithm must find $\bx$, which requires $\Omega(n)$ queries in expectation.
In the second construction, $X$ is a random sample of $n \simeq \exp(d/(1+\gamma)^2)$ points from the $d$-dimensional hypercube $\{0,1\}^d$ such that each coordinate is independently $1$ with probability $\simeq \frac{1}{1+\gamma}$.
Similarly to the first construction we show that, for $\bx \in X$ drawn uniformly at random, setting $W=(1+\gamma)\diag(x_1,\ldots,x_d)$ makes $\bx$ an outlier, and any algorithm needs $\Omega(n)$ queries to find it.

\section{Experiments} 
\label{sec:exp}
We implemented our algorithm \AlgoFULL\ and compared it against \scq-$k$-means \cite{ashtiani2016clustering}.
To this end, we generated four synthetic instances on $n=10^5$ points with increasing dimension $d=2,4,6,8$.
The latent clusterings consist of $k=5$ ellipsoidal clusters of equal size, each one with margin $\gamma=1$ w.r.t.\ a random center and a random PSD matrix with condition number $\kappa = 100$, making each cluster stretched by $10\times$ in a random direction.
To account for an imperfect knowledge of the data, we fed \AlgoFULL\ with a value of $\gamma = 10$ (thus, it could in principle output a wrong clustering).
We also adopted for \AlgoFULL\ the batch sampling of \scq-$k$-means, i.e., we draw $k \cdot 10$ samples in each round; this makes \AlgoFULL\ slightly less efficient than with its original sampling scheme (see line \ref{line:A2sample}).

To further improve the performance of \AlgoFULL, we use a simple ``greedy hull expansion'' heuristic that can increase the number of points recovered at each round without performing additional queries.
Immediately after taking the sample $S_C$, we repeatedly expand its convex hull $\conv(S_C)$ by a factor $\simeq(1+\nicefrac{\gamma}{d})$, and add all the points that fall inside it to $S_C$.
If $C$ is sufficiently dense, a substantial fraction of it will be added to $S_C$; while, by the margin assumption, no point outside $C$ will ever be added to $S_C$ (see the proof of the tessellation).
This greedy hull expansion is repeated until no new points are found, in which case we proceed to compute the MVEE and the tessellation.

Figure~\ref{fig:exp1} shows for both algorithms the clustering error $\ErrClust$ versus the number of queries, round by round, averaged over $10$ independent runs (\scq-$k$-means has a single measurement since it runs ``in one shot'').
The run variance is negligible and we do not report it.
Observe that the error of \scq-$k$-means is always in the range 20\%--40\%.
In contrast, the error of \AlgoFULL\ decreases exponentially with the rounds until the latent clustering is exactly recovered, as predicted by our theoretical results.
To achieve $\ErrClust \le .05$, \AlgoFULL\  uses less than $3\%$ of the queries needed by a brute force labeling, which is $kn = 5\times 10^5$.
Note that, except when clusters are aligned as in Figure~\ref{fig:3clust}, \scq-$k$-means continues to perform poorly even after whitening the input data to compensate for skewness.
Finally, note how the number of queries issued by \AlgoFULL\ increases with the dimensionality $d$, in line with Theorem~\ref{thm:LB}.
\newcommand{\figw}{2.6in}
\newcommand{\figh}{2.2in}
\begin{figure}[h!]
\centering
\begin{subfigure}{0.48\textwidth}
\centering
\includegraphics[width=\figw,height=\figh]{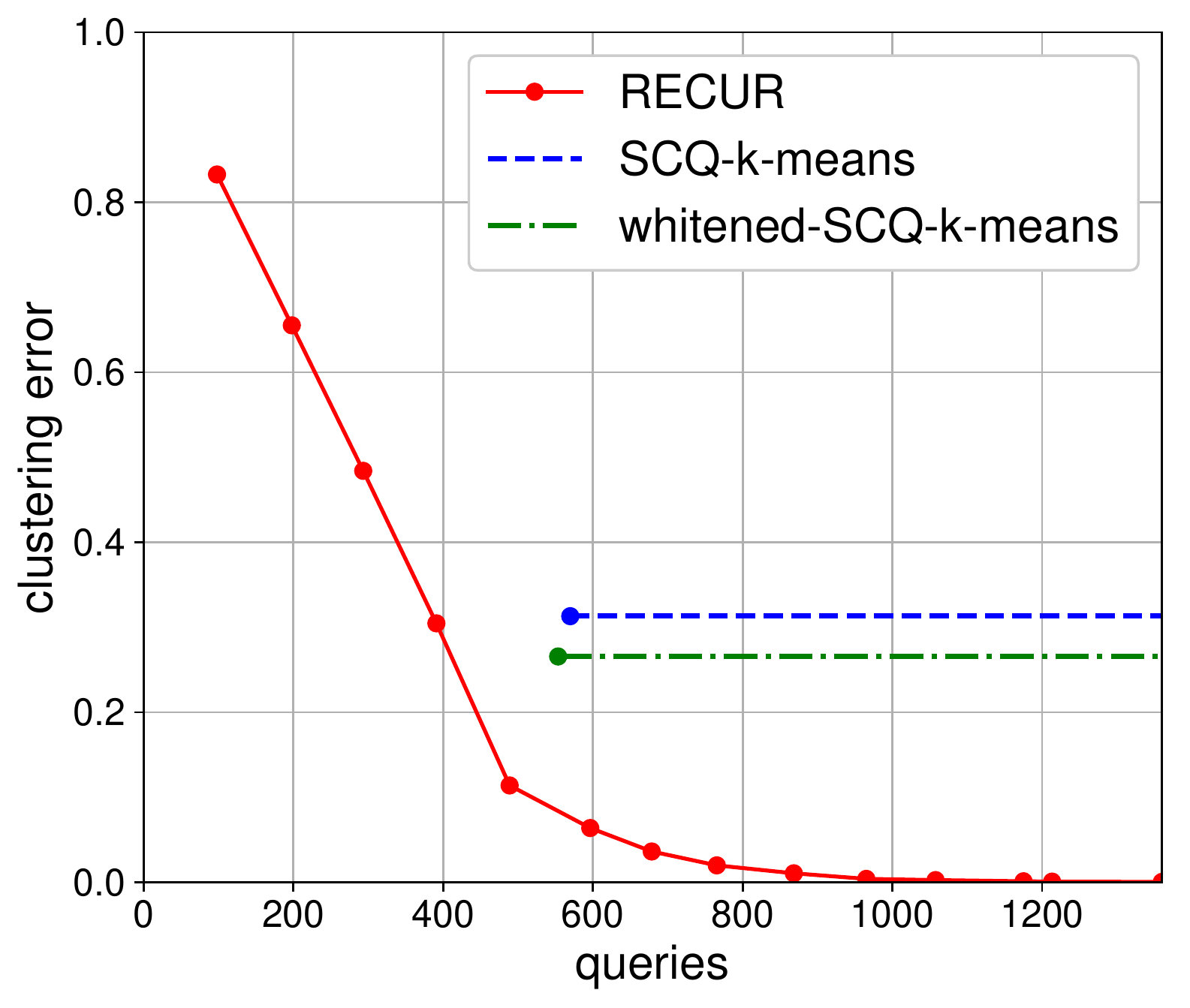}
\end{subfigure}
\hfill
\begin{subfigure}{0.48\textwidth}
\centering
\includegraphics[width=\figw,height=\figh]{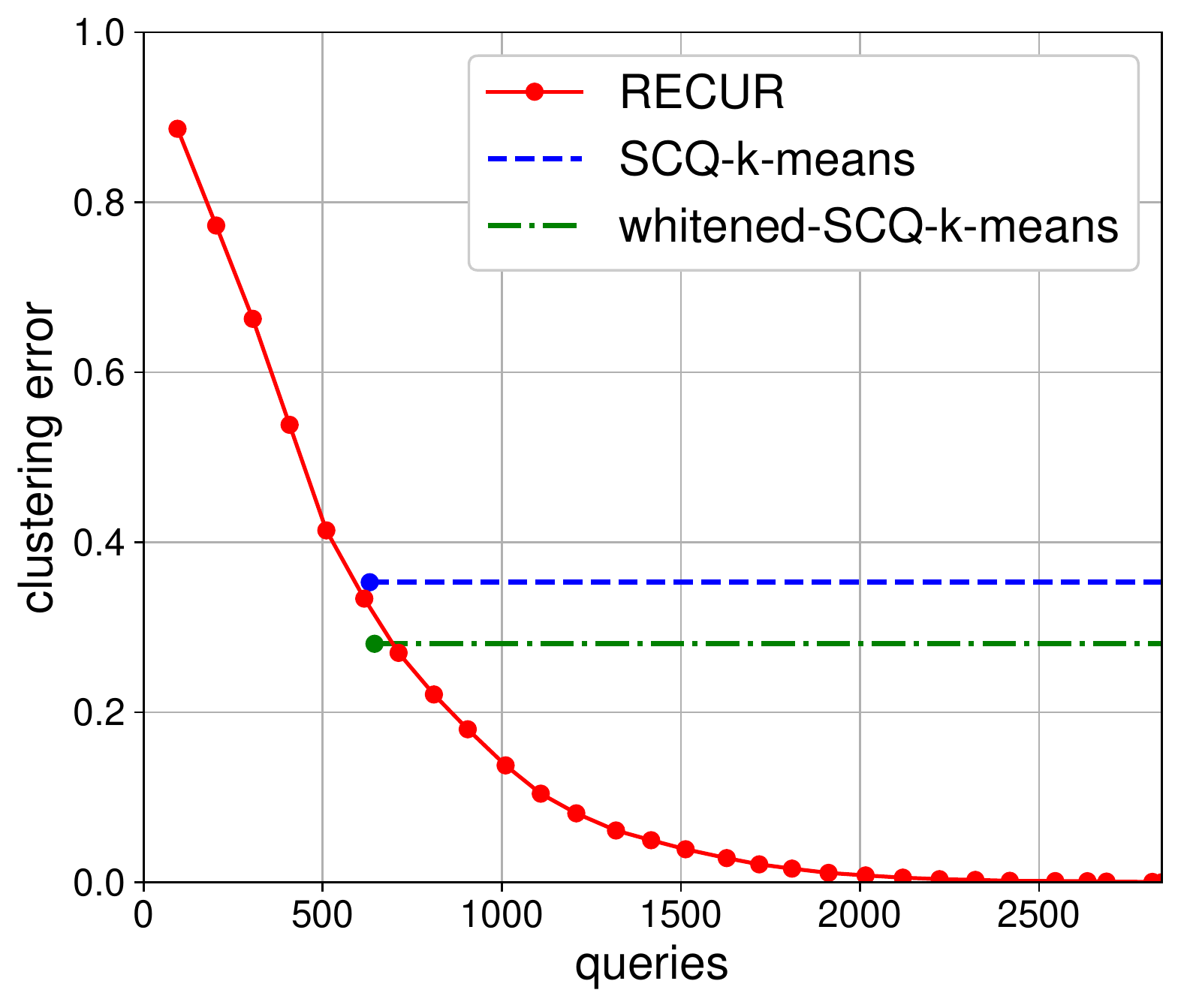}
\end{subfigure}
\\[1mm]
\begin{subfigure}{0.48\textwidth}
\centering
\includegraphics[width=\figw,height=\figh]{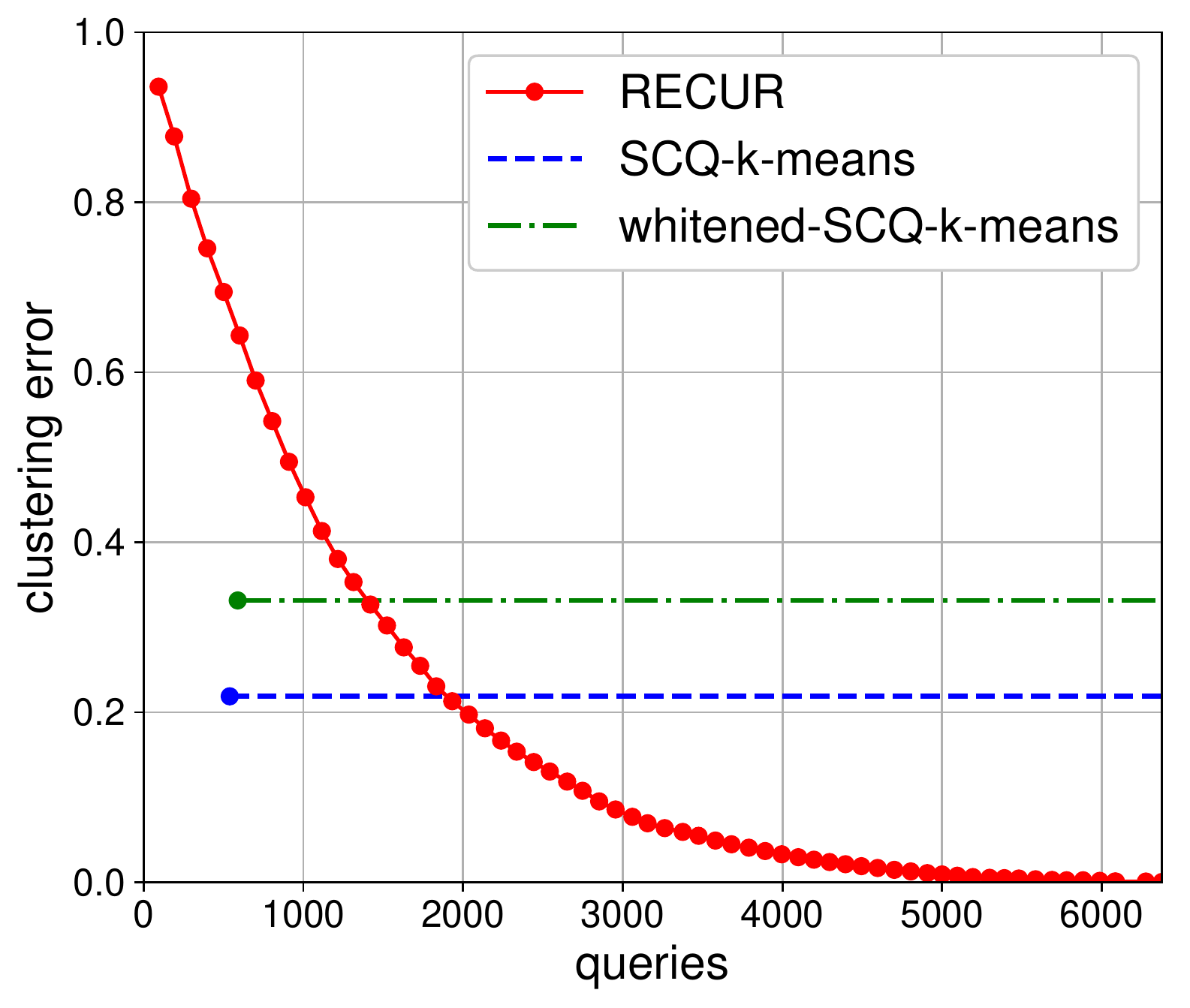}
\end{subfigure}
\hfill
\begin{subfigure}{0.48\textwidth}
\centering
\includegraphics[width=\figw,height=\figh]{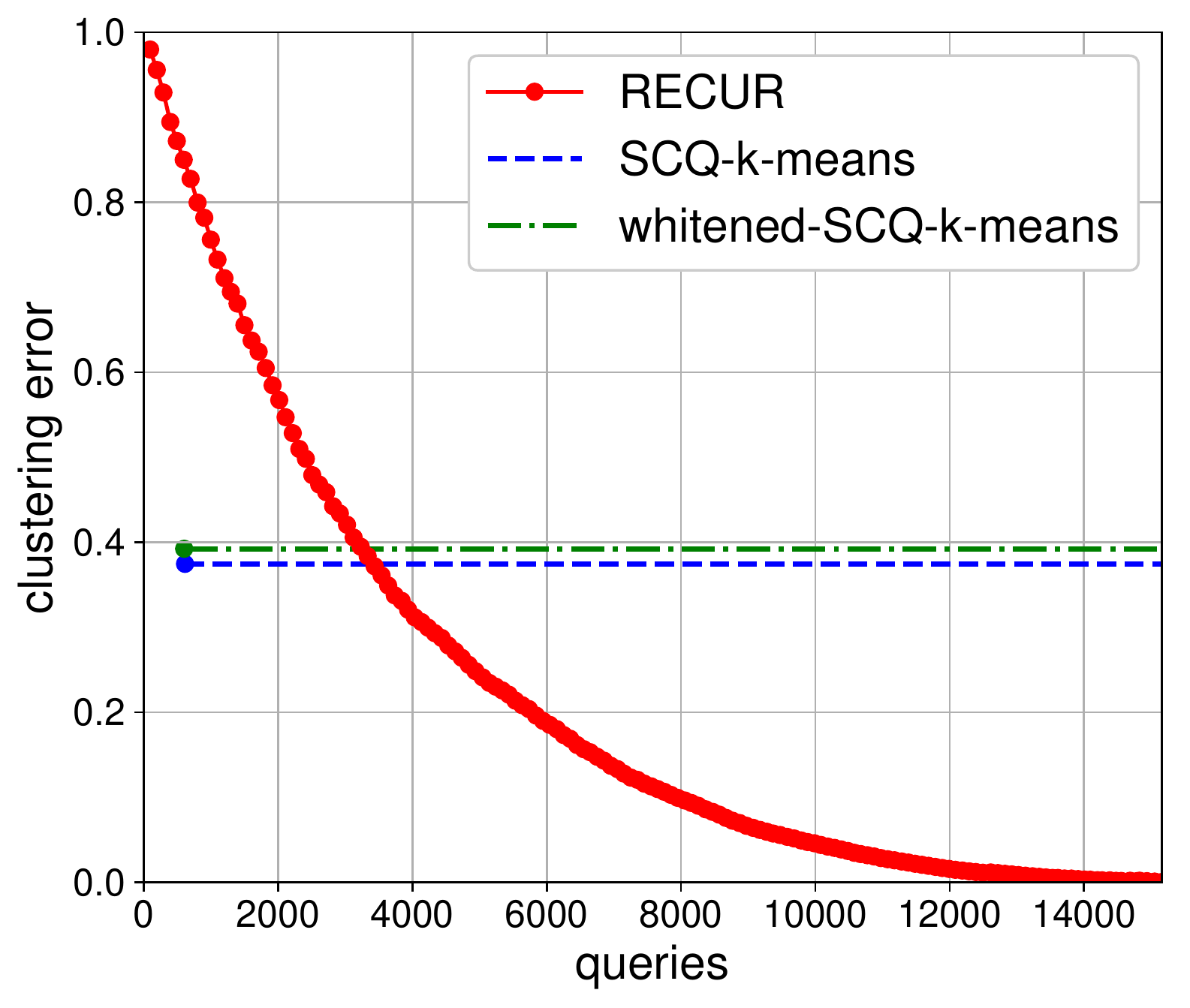}
\end{subfigure}
\caption{Clustering error vs.\ number of queries for $k=5$ and $d=2,4,6,8$ (left to right, top to bottom).
While \scq-$k$-means performs rather poorly, \AlgoFULL\ always achieves exact reconstruction.
}
\label{fig:exp1}
\end{figure}

\section{Conclusions}
We have given a novel technique that, under general conditions, allows one to actively recover a clustering using only $\scO(\ln n)$ same-cluster queries.
Unlike previous work, our technique is robust to distortions and manglings of the clusters, and works for arbitrary clusterings rather than only for those based on the solution of an optimization problem.
Our work leaves open three main questions:\\
\textbf{Q1:} Can our assumptions be strengthened in order to reduce the dependence on the dimension from exponential to polynomial, but without falling back to the setting of Ashtiani et al.~\cite{ashtiani2016clustering}?
\\
\textbf{Q2:} Can our assumptions be further relaxed, for instance by assuming a class of transformations more general than those given by PSD matrices?
\\
\textbf{Q3:} Is there a natural and complete characterization of the class of clusterings that can be reconstructed with $\scO(\ln n)$ queries?

\section*{Acknowledgements}
The authors gratefully acknowledge partial support by the Google Focused Award ``Algorithms and Learning for AI'' (ALL4AI).
Marco Bressan was also supported in part by the ERC Starting Grant DMAP 680153 and by the ``Dipartimenti di Eccellenza 2018-2022'' grant awarded to the Department of Computer Science of the Sapienza University of Rome.
Nicolò Cesa-Bianchi is also supported by the MIUR PRIN grant Algorithms, Games, and Digital Markets (ALGADIMAR) and by the EU Horizon 2020 ICT-48 research and innovation action under grant agreement 951847, project ELISE (European Learning and Intelligent Systems Excellence).

\section*{Broader impact}
This work does not present any foreseeable societal consequence.

\bibliographystyle{plain}
\bibliography{biblio}

\begin{thebibliography}{10}

\bibitem{ahmadian2019better}
Sara Ahmadian, Ashkan Norouzi-Fard, Ola Svensson, and Justin Ward.
\newblock Better guarantees for $k$-means and euclidean $k$-median by
  primal-dual algorithms.
\newblock {\em SIAM Journal on Computing}, 2019.
\newblock To appear.

\bibitem{ailon2018approximate}
Nir Ailon, Anup Bhattacharya, and Ragesh Jaiswal.
\newblock Approximate correlation clustering using same-cluster queries.
\newblock In {\em Proc.\ of LATIN}, pages 14--27, 2018.

\bibitem{ailon2017approximate}
Nir Ailon, Anup Bhattacharya, Ragesh Jaiswal, and Amit Kumar.
\newblock Approximate clustering with same-cluster queries.
\newblock In {\em Proc.\ of ITCS}, volume~94, pages 40:1--40:21, 2018.

\bibitem{ashtiani2016clustering}
Hassan Ashtiani, Shrinu Kushagra, and Shai Ben-David.
\newblock Clustering with same-cluster queries.
\newblock In {\em Advances in Neural Information Processing Systems 29}, pages
  3216--3224. 2016.

\bibitem{awasthi2012center}
Pranjal Awasthi, Avrim Blum, and Or~Sheffet.
\newblock Center-based clustering under perturbation stability.
\newblock {\em Information Processing Letters}, 112(1-2):49--54, 2012.

\bibitem{balcan2007margin}
Maria-Florina Balcan, Andrei Broder, and Tong Zhang.
\newblock Margin based active learning.
\newblock In {\em Proc.\ of COLT}, pages 35--50, 2007.

\bibitem{balcan2013active}
Maria-Florina Balcan and Phil Long.
\newblock Active and passive learning of linear separators under log-concave
  distributions.
\newblock In {\em Proc.\ of COLT}, pages 288--316, 2013.

\bibitem{dasgupta1999learning}
Sanjoy Dasgupta.
\newblock Learning mixtures of {G}aussians.
\newblock In {\em Proc.\ of IEEE FOCS}, page 634, 1999.

\bibitem{dasgupta2005analysis}
Sanjoy Dasgupta.
\newblock Analysis of a greedy active learning strategy.
\newblock In {\em Advances in Neural Information Processing Systems 17}, pages
  337--344, 2005.

\bibitem{Dubhashi&2009}
Devdatt Dubhashi and Alessandro Panconesi.
\newblock {\em Concentration of Measure for the Analysis of Randomized
  Algorithms}.
\newblock Cambridge University Press, New York, NY, USA, 1st edition, 2009.

\bibitem{dzogang2012ellipsoidal}
Fabon Dzogang, Christophe Marsala, Marie-Jeanne Lesot, and Maria Rifqi.
\newblock An ellipsoidal $k$-means for document clustering.
\newblock In {\em Proc.\ of IEEE ICDM}, pages 221--230, 2012.

\bibitem{emamjomeh2018adaptive}
Ehsan Emamjomeh-Zadeh and David Kempe.
\newblock Adaptive hierarchical clustering using ordinal queries.
\newblock In {\em Proc.\ of ACM-SIAM SODA}, pages 415--429. SIAM, 2018.

\bibitem{firmani2018robust}
Donatella Firmani, Sainyam Galhotra, Barna Saha, and Divesh Srivastava.
\newblock Robust entity resolution using a crowd oracle.
\newblock {\em IEEE Data Eng. Bull.}, 41(2):91--103, 2018.

\bibitem{Teytaud11}
Herv\'{e} Fournier and Olivier Teytaud.
\newblock Lower bounds for comparison based evolution strategies using
  vc-dimension and sign patterns.
\newblock {\em Algorithmica}, 59(3):387–408, March 2011.

\bibitem{gamlath2018semi}
Buddhima Gamlath, Sangxia Huang, and Ola Svensson.
\newblock Semi-supervised algorithms for approximately optimal and accurate
  clustering.
\newblock In {\em Proc.\ of ICALP}, pages 57:1--57:14, 2018.

\bibitem{Gonen&2013}
Alon Gonen, Sivan Sabato, and Shai Shalev-Shwartz.
\newblock Efficient active learning of halfspaces: an aggressive approach.
\newblock {\em The Journal of Machine Learning Research}, 14(1):2583--2615,
  2013.

\bibitem{gonzalez1985clustering}
Teofilo~F Gonzalez.
\newblock Clustering to minimize the maximum intercluster distance.
\newblock {\em Theoretical Computer Science}, 38:293--306, 1985.

\bibitem{gruenheid2015fault}
Anja Gruenheid, Besmira Nushi, Tim Kraska, Wolfgang Gatterbauer, and Donald
  Kossmann.
\newblock Fault-tolerant entity resolution with the crowd.
\newblock {\em CoRR}, abs/1512.00537, 2015.

\bibitem{hanneke2015minimax}
Steve Hanneke and Liu Yang.
\newblock Minimax analysis of active learning.
\newblock {\em The Journal of Machine Learning Research}, 16(1):3487--3602,
  2015.

\bibitem{hardt2015tight}
Moritz Hardt and Eric Price.
\newblock {Tight Bounds for Learning a Mixture of Two {G}aussians}.
\newblock In {\em Proc.\ of ACM STOC}, page 753–760, 2015.

\bibitem{huleihel2019same}
Wasim Huleihel, Arya Mazumdar, Muriel M{\'e}dard, and Soumyabrata Pal.
\newblock Same-cluster querying for overlapping clusters.
\newblock In {\em Advances in Neural Information Processing Systems 32}, pages
  10485--10495, 2019.

\bibitem{kalai2010efficiently}
Adam~Tauman Kalai, Ankur Moitra, and Gregory Valiant.
\newblock Efficiently learning mixtures of two {G}aussians.
\newblock In {\em Proc.\ of ACM STOC}, pages 553--562, 2010.

\bibitem{kameyama1999semiconductor}
Keisuke Kameyama and Yukio Kosugi.
\newblock Semiconductor defect classification using hyperellipsoid clustering
  neural networks and model switching.
\newblock In {\em Proc.\ of IJCNN'99 (Cat.\ No.\ 99CH36339)}, volume~5, pages
  3505--3510. IEEE, 1999.

\bibitem{Kane19}
D.~M. {Kane}, S.~{Lovett}, S.~{Moran}, and J.~{Zhang}.
\newblock Active classification with comparison queries.
\newblock In {\em Proc.\ of IEEE FOCS}, pages 355--366, 2017.

\bibitem{KhachiyanMVE}
Leonid~G. Khachiyan.
\newblock Rounding of polytopes in the real number model of computation.
\newblock {\em Mathematics of Operations Research}, 21(2):307--320, 1996.

\bibitem{kulis2013metric}
Brian Kulis.
\newblock Metric learning: A survey.
\newblock {\em Foundations and Trends in Machine Learning}, 5(4):287--364,
  2013.

\bibitem{li2016approximating}
Shi Li and Ola Svensson.
\newblock Approximating $k$-median via pseudo-approximation.
\newblock {\em SIAM Journal on Computing}, 45(2):530--547, 2016.

\bibitem{marica2014hyper}
Vasile-George Marica.
\newblock Hyper-ellipsoid clustering of time series: A case study for daily
  stock returns.
\newblock {\em Procedia Economics and Finance}, 15:777--783, 2014.

\bibitem{mazumdar2017semisupervised}
Arya Mazumdar and Soumyabrata Pal.
\newblock {Semisupervised Clustering, AND-Queries and Locally Encodable Source
  Coding}.
\newblock In {\em Advances in Neural Information Processing Systems 30}, pages
  6489--6499, 2017.

\bibitem{mazumdar2017clustering}
Arya Mazumdar and Barna Saha.
\newblock Clustering with noisy queries.
\newblock In {\em Advances in Neural Information Processing Systems 30}, pages
  5788--5799, 2017.

\bibitem{NIPS2017_7054}
Arya Mazumdar and Barna Saha.
\newblock Query complexity of clustering with side information.
\newblock In I.~Guyon, U.~V. Luxburg, S.~Bengio, H.~Wallach, R.~Fergus,
  S.~Vishwanathan, and R.~Garnett, editors, {\em Advances in Neural Information
  Processing Systems 30}, pages 4682--4693. Curran Associates, Inc., 2017.

\bibitem{moshtaghi2011clustering}
Masud Moshtaghi, Timothy~C Havens, James~C Bezdek, Laurence Park, Christopher
  Leckie, Sutharshan Rajasegarar, James~M Keller, and Marimuthu Palaniswami.
\newblock Clustering ellipses for anomaly detection.
\newblock {\em Pattern Recognition}, 44(1):55--69, 2011.

\bibitem{MR95}
Rajeev Motwani and Prabhakar Raghavan.
\newblock {\em Randomized Algorithms}.
\newblock Cambridge University Press, USA, 1995.

\bibitem{nowak2011geometry}
Robert~D Nowak.
\newblock The geometry of generalized binary search.
\newblock {\em IEEE Transactions on Information Theory}, 57(12):7893--7906,
  2011.

\bibitem{saha2019correlation}
Barna Saha and Sanjay Subramanian.
\newblock Correlation clustering with same-cluster queries bounded by optimal
  cost.
\newblock {\em CoRR}, abs/1908.04976, 2019.

\bibitem{sanyal2019semi}
Deepayan Sanyal and Swagatam Das.
\newblock On semi-supervised active clustering of stable instances with
  oracles.
\newblock {\em Information Processing Letters}, 151:105833, 2019.

\bibitem{shalevshwartz2014understanding}
Shai Shalev-Shwartz and Shai Ben-David.
\newblock {\em Understanding Machine Learning: From Theory to Algorithms}.
\newblock Cambridge University Press, USA, 2014.

\bibitem{shuhua2013ellipsoids}
Ma~Shuhua, Wang Jinkuan, and Liu Zhigang.
\newblock Ellipsoids clustering algorithm based on the hierarchical division
  for {WSN}s.
\newblock In {\em Proc.\ of IEEE CCC}, pages 7394--7397. IEEE, 2013.

\bibitem{todd2016minimum}
Michael~J. Todd.
\newblock {\em Minimum-Volume Ellipsoids}.
\newblock SIAM, Philadelphia, PA, 2016.

\bibitem{verroios2015entity}
Vasilis Verroios and Hector Garcia-Molina.
\newblock Entity resolution with crowd errors.
\newblock In {\em Proc.\ of IEEE ICDE}, pages 219--230, 2015.

\bibitem{verroios2017waldo}
Vasilis Verroios, Hector Garcia-Molina, and Yannis Papakonstantinou.
\newblock Waldo: An adaptive human interface for crowd entity resolution.
\newblock In {\em Proc.\ of ACM SIGMOD}, pages 1133--1148, 2017.

\bibitem{HDPbook}
Roman Vershynin.
\newblock {\em High-Dimensional Probability: An Introduction with Applications
  in Data Science}.
\newblock Cambridge Series in Statistical and Probabilistic Mathematics.
  Cambridge University Press, 2018.

\bibitem{vitale2019flattening}
Fabio Vitale, Anand Rajagopalan, and Claudio Gentile.
\newblock Flattening a hierarchical clustering through active learning.
\newblock In {\em Advances in Neural Information Processing Systems 32}, pages
  15263--15273, 2019.

\end{thebibliography}

\tikzstyle{dot}=[draw,fill,shape=circle,inner sep=0pt,minimum size=3pt]
\tikzstyle{ps}=[circle,draw, fill=black, minimum size=4,inner sep=0pt, outer sep=0pt]
\tikzstyle{ns}=[circle,draw, fill=white, minimum size=4,inner sep=0pt, outer sep=0pt]
\tikzstyle{mve}=[circle,draw,densely dotted,thick]
\tikzstyle{ell}=[circle,draw,thick]

\appendix

\section{Ancillary results}

\subsection{VC-dimension of ellipsoids}
\label{sub:VCofEL}
For any PSD matrix $M$, we denote by $\EL_{M} = \theset{\bx \in \R^d}{d_M(\bx,\bmu) \le 1}$ the $\bmu$-centered ellipsoid with semiaxes of length $\lambda_1^{-1/2},\ldots,\lambda_d^{-1/2}$, where $\lambda_1,\ldots,\lambda_d \ge 0$ are the eigenvalues of $M$.
We recall the following classical VC-dimension bound (see, e.g., \cite{Teytaud11}).
\begin{theorem}
\label{thm:vcE}
The VC-dimension of the class $\Hs = \{\EL_M \,:\, M \in \R^{d}, M \succeq 0 \}$ of (possibly degenerate) ellipsoids in $\R^d$ is $\frac{d^2+3d}{2}$.
\end{theorem}

\subsection{Generalization error bounds}
The next result is a simple adaptation of the classical VC bound for the realizable case (see, e.g., \cite[Theorem 6.8]{shalevshwartz2014understanding}).
%
\begin{theorem}
\label{thm:vc+pac}
There exists a universal constant $c > 0$ such that for any family $\Hs$ of measurable sets $E \subset \R^d$ of VC-dimension $d < \infty$, any probability distribution $\mathcal{D}$ on $\R^d$, and any $\epsilon,\delta \in \left(0, 1\right)$, if $S$ is a sample of $m \ge c \frac{d\ln(1/\epsilon)+\ln(1/\delta)}{\epsilon}$ points drawn i.i.d.\ from $\mathcal{D}$, then for any $E^*\in\Hs$ we have:
\begin{align*}
    \mathcal{D}\big(E\SymDif E^*\big) \le \epsilon
\qquad\text{and}\qquad
    \mathcal{D}\big(E'\setminus E^*\big) \le \epsilon
\end{align*}
with probability at least $1-\delta$ with respect to the random draw of $S$, where $E$ is any element of $\Hs$ such that $E \cap S = E^* \cap S$, and $E'$ is any element of $\Hs$ such that $E^* \cap S \subseteq E' \cap S$.
\end{theorem}
The first inequality is the classical PAC bound for the zero-one loss, which uses the fact that the VC dimension of $\theset{E \SymDif E^*}{E\in\Hs}$ is the same as the VC dimension of $\Hs$. The second inequality follows immediately from the same proof by noting that, for any $E^*\in\Hs$ the VC dimension of $\theset{E\setminus E^*}{E\in\Hyp}$ is not larger than the VC dimension of $\Hs$ because, for any sample $S$ and for any $F,G\in\Hs$, $(F\setminus E^*) \cap S \neq (G \setminus E^*) \cap S$ implies $F \cap S \neq G \cap S$.

\subsection{Concentration bounds}
We recall standard concentration bounds for non-positively correlated binary random variables, see~\cite{Dubhashi&2009}.
Let $X_1,\ldots,X_n$ be binary random variables. We say that $X_1,\ldots,X_n$ are non-positively correlated if for all $I \subseteq \{1,\ldots,n\}$ we have:
\begin{align}
\Pr\big(\forall i \in I: X_i=0\big) \leq \prod_{i \in I} \Pr(X_i=0) \quad \text{and} \quad
\Pr\big(\forall i \in I: X_i=1\big) \leq \prod_{i \in I} \Pr(X_i=1)
\end{align}
\begin{lemma}[Chernoff bounds]
\label{lem:chernoff}
Let $X_1,\ldots,X_n$ be non-positively correlated binary random variables. Let $a_1,\ldots,a_n \in [0,1]$ and $X=\sum_{i=1}^{n}a_iX_i$. Then, for any $\epsilon > 0$, we have:
\begin{align}
\Pr\big(X < (1-\epsilon)\E[X]\big) &< e^{-\frac{\epsilon^2}{2}\E[X]} \\
\Pr\big(X > (1+\epsilon)\E[X]\big) &< e^{-\frac{\epsilon^2}{2+\epsilon}\E[X]} 
\end{align}
\end{lemma}

\subsection{Yao's minimax principle}
\label{sub:yao}
We recall Yao's minimax principle for Monte Carlo algorithms.
Let $\mathcal{A}$ be a finite family of deterministic algorithms and $\mathcal{I}$ a finite family of problem instances.
Fix any two distributions $\bp$ over $\mathcal{I}$ and $\bq$ over  $\mathcal{A}$, and any $\delta \in [0,\nicefrac{1}{2}]$.
Let $\min_{A \in \mathcal{A}} \E_{I \sim \bp}[C_{\delta}(I,A)]$ be the minimum, over every algorithm $A$ that fails with probability at most $\delta$ over the input distribution $\bp$, of the expect cost of $A$ over the input distribution itself.
Similarly, let $\max_{I \in \mathcal{I}} \E_{A \sim \bq}[C_{\delta}(I,A)]$ be the expected cost of the randomized algorithm defined by $\bq$ under its worst input from $\mathcal{I}$, assuming it fails with probability at most $\delta$.
Then (see~\cite{MR95}, Proposition 2.6):
\begin{align}
    \max_{I \in \mathcal{I}} \E_{\bq}[C_{\delta}(I,A)] \ge \frac{1}{2}\min_{A \in \mathcal{A}} \E_{\bp}[C_{2\delta}(I,A)]
\end{align}

 
\section{Supplementary material for Section~\ref{sec:single}}

\subsection{Monochromatic Tessellation}
\label{sub:proof:R}
We give a formal version of the claim about the monochromatic tessellation of Section~\ref{sec:single}:
\begin{theorem}
\label{thm:R}
Suppose we are given an ellipsoid $\EL$ such that $\frac{1}{d \strtch}\EL \subset \conv(S_C) \subset E$ for some stretch factor $\strtch > 0$.
Then for a suitable choice of $\beta_i,\rho,b$, the tessellation $\mathcal{R}$ of the positive orthant of $E$ (Definition~\ref{def:R}) satisfies:
\begin{enumerate}[label={(\arabic*)},topsep=0pt,parsep=0pt,itemsep=1pt]
\item $|\mathcal{R}| \le \max\big\{1,O\big(\frac{d \strtch}{\gamma}\ln\!\frac{d \strtch}{\gamma}\big)^d\big\}$
\item $E \cap \R_+^d \subseteq \cup_{R \in \mathcal{R}}R$ \phantom{$\frac{d^3}{\gamma^4}$}
\item for every $R \in \mathcal{R}$, the set $R \cap \ELOUT$ is monochromatic \phantom{$\frac{d^3}{\gamma^4}$}
\end{enumerate}
\end{theorem}
In order to prove Theorem~\ref{thm:R}, we define the tessellation and prove properties (1-3) for $\gamma \le \nicefrac{1}{2}$.
For $\gamma > \frac{1}{2}$ the tessellation is defined as for $\gamma=\frac{1}{2}$, and one can check all properties still hold.
In the proof we use a constant $c=\sqrt{5}$ and assume $\gamma < c^2-2c$, which is satisfied since $c^2-2c = 5 - 2\sqrt{5} > \nicefrac{1}{2}$.

First of all, we define the intervals $T_i$.
The base $i$-th coordinate is:
\begin{align}
\beta_i = \frac{\gamma}{c\sqrt{2d}} \frac{L_i}{\strtch d}
\label{eqn:betai}
\end{align}
Note that, for all $i$,
\begin{align}
\frac{L_i}{\beta_i} = \frac{ \strtch c d\sqrt{2 d}}{\gamma}
\label{eq:Libi}
\end{align}
Define:
\begin{align}
\alpha = \frac{\gamma}{c\sqrt{2} \strtch d}
\end{align}
and let:
\begin{align}
b = \max\left(0,\left\lceil \log_{1+\alpha}\Big( \frac{ c \strtch d\sqrt{2d}}{\gamma}
\Big)\right\rceil\right)
\label{eqn:b}
\end{align}
(The parameter $\rho$ of the informal description of Section~\ref{sec:single} is exactly $1+\alpha$).
Finally, define the interval set along the $i$-th axis as:
\begin{align}
T_i =
\left\{\begin{array}{lc}
\big\{\big[0, \beta_i \big]\big\}     &  \text{if } b=0\\[10pt]
   \big\{
\big[0, \beta_i \big],
\big(\beta_i, \beta_i (1+\alpha)\big],
\ldots,
\big( \beta_i (1+\alpha)^{b-1}, \beta_i (1+\alpha)^{b} \big]
\big\}
  & \text{if } b\ge 1
\end{array}
\right.
\label{eqn:Ti}
\end{align}

\textbf{Proof of (1).}
By construction, $|T_i|=b+1$.
Thus, $|\mathcal{R}|=\prod_{i \in [d]}|T_i|=(b+1)^d$.
Thus, if $b = 0$ then $|\mathcal{R}|=1$, else by~\eqref{eqn:b} and~\ref{eq:Libi},
\begin{align}
    b &= \left\lceil \frac{\ln\!\big( \frac{c \strtch d\sqrt{2d}}{\gamma} \big)}{\ln(1+\alpha)}\right\rceil 
    \\&\le \left\lceil \frac{2}{\alpha}\ln\!\Big( \frac{c \strtch d\sqrt{2d}}{\gamma} \Big)\right\rceil &&\text{since $\ln(1+\alpha) \ge \nicefrac{\alpha}{2}$ as $\alpha \le 1$}
    \\&= \left\lceil \frac{2\sqrt{2}c\strtch d}{\gamma}\ln \frac{c \strtch d\sqrt{2d}}{\gamma}\right\rceil && \text{definition of $\alpha$}
    \\&= O\!\left(\frac{d \strtch}{\gamma}\ln\frac{d \strtch}{\gamma}\right) && \text{since $d \strtch \ge 1, \gamma \le \nicefrac{1}{2}$}
\end{align}
in which case $|\mathcal{R}|=O\big(\frac{d \strtch}{\gamma}\ln\frac{d \strtch}{\gamma}\big)^d$.
Taking the maximum over the two cases proves the claim.

\textbf{Proof of (2).} We show for any $\bx \in \ELOUT \cap \R_+^d$ there exists $R \in \mathcal{R}$ containing $\bx$.
Clearly, if $\bx \in \ELOUT \cap \R_+^d$, then $\dotp{\bx,\bu_i} \in [0, L_i]$ for all $i \in [d]$.
But $T_i$ covers, along the $i$-th direction $\bu_i$, the interval from $0$ to
\begin{align}
\beta_i (1+\alpha)^{b}
= \beta_i (1+\alpha)^{\max(0,\lceil \log_{1+\alpha} (\nicefrac{L_i}{\beta_i}) \rceil)}
\ge \beta_i(1+\alpha)^{\lceil \log_{1+\alpha} (\nicefrac{L_i}{\beta_i}) \rceil} \ge L_i
\end{align}
Therefore some $R \in \mathcal{R}$ contains $\bx$.

\textbf{Proof of (3).}
Given any hyperrectangle $R \in \mathcal{R}$, we show that the existence of $\bx, \by \in R \cap \ELOUT$ with $\bx \in C$ and $\by \notin C$ leads to a contradiction.
For the sake of the analysis we conventionally set the origin at the center $\bmu$ of $\EL$, i.e.\ we assume $\bmu=\orig$.

We define $\ELIN=\frac{1}{\strtch d}\ELOUT$ and let $M = U\Lambda U^{\top}$ be its PSD matrix, where $U = \big[\bu_1,\ldots,\bu_d]$ and $\Lambda=\diag(\lambda_1,\ldots,\lambda_d)$.
Note that $\lambda_i = \frac{1}{\ell_i^2} = \frac{\strtch^2 d^2}{L_i^2}$ where $\ell_i=\frac{L_i}{\strtch d}$ is the length of the $i$-th semiaxis of $\ELIN$.
For any $R \in \mathcal{R}$, let $R_i$ be the projection of $R$ on $\bu_i$ (i.e.\ $R_i$ is one of the intervals of $T_i$ defined in~\eqref{eqn:Ti}).
Let $D = D(R) = \{ i \in [d] : 0 \notin R_i\}$.
We let $U_{D}$ and $U_{\neg D}$ be the matrices obtained by zeroing out the columns of $U$ corresponding to the indices in $[d]\setminus D$ and $D$, respectively.
Observe that if $\bx,\by \in R \cap \ELOUT$ then:
\begin{alignat}{2}
\dotp{\bx-\by,\bu_i}^2 &< \alpha^2 \dotp{\bx,\bu_i}^2 \quad && \forall i \in D \label{eqn:Rbound1}
\\
\dotp{\bx-\by,\bu_i}^2 &\le \beta_i^2 \quad && \forall i \notin D \label{eqn:Rbound2}
\end{alignat}
Now suppose $C$ has margin at least $\gamma$ for some $\gamma \in (0, c^2-2c]$, and suppose $\bx,\by \in R \cap E$ with $\bx \in C$ and $\by \notin C$.
Through a set of ancillary lemmata proven below, this leads to the absurd:
\begin{align}
\frac{\gamma^2}{c^2} &< d_{W}(\by,\bx)^2 &&\text{Lemma~\ref{lem:ancil1}}
\\& \le d_{M}(\by,\bx)^2 &&\text{Lemma~\ref{lem:WprecM}}
\\& < \alpha^2 d_{M}(\bx,\bmu)^2 + \frac{\gamma^2}{2c^2} &&\text{Lemma~\ref{lem:ancil2}}
\\& \le \frac{\gamma^2}{2c^2} + \frac{\gamma^2}{2c^2} &&\text{Lemma~\ref{lem:ancil3}}
\end{align}
In the rest of the proof we prove the four lemmata.
\begin{lemma}
\label{lem:ancil1}
$\frac{\gamma}{c} < d_{W}(\by,\bx)$.
\end{lemma}
\begin{proof}
Let $\bz$ be the point w.r.t.\ which the margin of $C$ holds.
By the margin assumption,
\begin{align}
d_{W}(\by,\bz) > \sqrt{1+\gamma} \qquad \text{and} \qquad d_{W}(\bx,\bz) \le 1
\end{align}
By the triangle inequality then,
\begin{align}
d_{W}(\by,\bx) \ge d_{W}(\by,\bz) - d_{W}(\bx,\bz) > \sqrt{1+\gamma}-1
\end{align}
One can check that for $\gamma \le c^2 -2c$ we have $1+\gamma \ge (1+\frac{\gamma}{c})^2$.
Therefore
\begin{align}
d_{W}(\by,\bx) > \sqrt{(1 + \nicefrac{\gamma}{c})^2}-1 = \frac{\gamma}{c}
\end{align}
as desired.
\end{proof}
\begin{lemma}
\label{lem:WprecM}
$d_{W}(\cdot) \le d_{M}(\cdot)$.
\end{lemma}
\begin{proof}
By the assumptions of the theorem, $\ELIN \subseteq \conv_{\bmu}(C)$.
Moreover, by the assumptions on $d_W(\cdot)$, the unit ball of $d_W(\cdot)$ contains $\conv(C)$.
Thus, the unit ball of $d_W(\cdot)$ contains the unit ball of $d_M(\cdot)$.
This implies $W \preceq M$, thus $\norm{W}{\cdot} \le \norm{M}{\cdot}$ and $d_{W}(\cdot) \le d_{M}(\cdot)$.
\end{proof}
\begin{lemma}
\label{lem:ancil2}
$d_{M}(\by,\bx)^2 < \alpha^2 d_{M}(\bx,\bmu)^2 + \frac{\gamma^2}{2c^2}$.
\end{lemma}
\begin{proof}
We decompose $d_{M}(\by,\bx)^2$ along the colspaces of $U_D$ and $U_{\neg D}$:
\begin{align}
d_{M}(\by,\bx)^2
&= \norm{2}{M^{1/2}(\by-\bx)}^2
\\&= \norm{U_{D}U_{D}^{\top}}{M^{1/2}(\by-\bx)}^2 + \norm{U_{\neg D}U_{\neg D}^{\top}}{M^{1/2}(\by-\bx)}^2
\label{eq:d2M}
\end{align}
Next, we bound the two terms of~\eqref{eq:d2M}.
To this end, we need to show that for all $D \subseteq [d]$ and $\bv \in \R^d$:
\begin{align}
    \norm{U_D\tp{U_D}}{M^{1/2} \bv }^2
    = \sum_{i \in D} \lambda_i \dotp{\bv,\bu_i}^2
\end{align}
Let indeed $J_D=\diag(\bOne_{D})$ be the selection matrix corresponding to the indices of $D$.
Then $U_D=U J_D$, and so $\tp{U}U_D = \tp{U} U J_D = J_D$.
This gives:
\begin{align}
\norm{U_D \tp{U_D}}{M^{1/2} \bv }^2 &= \tp{\bv} (U \Lambda^{1/2} \tp{U}) U_D \tp{U_D} (U \Lambda^{1/2} \tp{U})  \bv && \text{definition of $M$ and $\norm{\cdot}{\cdot}$}
\\ &= \tp{\bv} U \Lambda^{1/2} J_D J_D \Lambda^{1/2} \tp{U}  \bv && \text{since $\tp{U}U_D=J_D$}
\\ &= \tp{\bv} U J_D \Lambda^{1/2} \Lambda^{1/2} J_D \tp{U}  \bv && \text{since $\Lambda,J_D$ are diagonal}
\\ &= \tp{\bv} U_D \Lambda \tp{U_D}  \bv &&\text{since $U J_D = U_D$}
\\ &= \norm{\Lambda}{\tp{U_D} \bv}^2 &&\text{by definition}
\\ & = \sum_{i \in D} \lambda_i \dotp{\bv,\bu_i}^2
\label{eqn:normM_to_normL}
\end{align}
Now we can bound the first term of~\eqref{eq:d2M}:
\begin{align}
\norm{U_{D}U_{D}^{\top}}{M^{1/2}(\by-\bx)}^2 
&= \sum_{i \in D} \lambda_i \dotp{\by-\bx,\bu_i}^2 &&\text{by~\eqref{eqn:normM_to_normL}}
\\&< \alpha^2 \sum_{i \in D} \lambda_i \dotp{\bx,\bu_i}^2 && \text{by~\eqref{eqn:Rbound1}}
\\&= \alpha^2 \norm{U_{D}U_{D}^{\top}}{M^{1/2}\bx}^2 &&\text{by~\eqref{eqn:normM_to_normL}}
\\&\le \alpha^2\norm{UU^{\top}}{M^{1/2}\bx}^2
\\&= \alpha^2\norm{2}{M^{1/2}\bx}^2 &&\text{since $UU^{\top}=I$}
\\&= \alpha^2 d_{M}^2(\bx,\bmu) && \text{since $\bmu=\orig$}
\end{align}
And for the second term of~\eqref{eq:d2M}, we have:
\begin{align}
\norm{U_{\neg D}U_{\neg D}^{\top}}{M^{1/2}(\by-\bx)}^2
&= \sum_{i \notin D} \lambda_i \dotp{\by-\bx, \bu_i}^2 && \text{by~\eqref{eqn:normM_to_normL}}
\\&\le \sum_{i \notin D} \lambda_i \beta_i^2 && \text{by \eqref{eqn:Rbound2}}
\\&= \sum_{i \notin D} \frac{\strtch^2 d^2}{L_i^2} \left(\frac{\gamma}{c\sqrt{2d}} \frac{L_i}{\strtch d}\right)^2 && \text{by definition of $\lambda_i$ and $\beta_i$}
\\&= \sum_{i \notin D} \frac{\gamma^2}{2dc^2} 
\\&\le \frac{\gamma^2}{2c^2} 
\end{align}
Summing the bounds on the two terms shows that $d_{M}(\by,\bx)^2 < \alpha^2 d_{M}(\bx,\bmu)^2 + \frac{\gamma^2}{2c^2}$, as claimed.
\end{proof}
\begin{lemma}
\label{lem:ancil3}
$\alpha^2 d_{M}(\bx,\bmu)^2 \le \frac{\gamma^2}{2c^2}$.
\end{lemma}
\begin{proof}
By construction we have $\bx \in \ELOUT$ and $\ELOUT = \strtch d \cdot \ELIN$.
Therefore $\frac{1}{\strtch d}\bx \in \ELIN$, that is:
\begin{align}
1 \ge d_M\Big(\frac{1}{\strtch d} \bx,\bmu\Big)^2
= \frac{1}{\strtch^2 d^2} d_M(\bx,\bmu)^2
\end{align}
where we used the fact that $d_M(\cdot,\bmu)^2=\norm{M}{\cdot}^2$ since $\bmu=\orig$.
Rearranging terms, this proves that $d_M(\bx,\bmu)^2 \le \strtch^2 d^2$.
Multiplying by $\alpha^2$, we obtain:
\begin{align}
\alpha^2 d_{M}(\bx,\bmu)^2 \le \left(\frac{\gamma}{\sqrt{2} c \strtch d}\right)^2 \strtch^2 d^2 = \frac{\gamma^2}{2c^2 }
\end{align}
as desired.
\end{proof}
The proof of the theorem is complete.

\subsection{Low-stretch separators and proof of Theorem~\ref{thm:stretch}}
\label{apx:stretch}
In this section we show how to compute the separator of Theorem~\ref{thm:stretch}.
In fact, computing the separator is easy; the nontrivial part is Theorem~\ref{thm:stretch} itself, that is, showing that such a separator always exists.

To compute the separator we first compute the MVEE $\MVE = (\Mstar,\bmustar)$ of $S_C$ (see Section~\ref{sec:single}). We then solve the following semidefinite program:
\begin{equation}
\label{eq:sdp}
\begin{aligned}
&\max_{\alpha\in\R,\bmu\in\R^d,M\in\R^{d \times d}} \, \alpha \\
\text{s.t.} \quad 
& M \succeq \alpha \,\Mstar
\\
& \big\langle M,(\bx-\bmu)(\bx-\bmu)^{\top} \big\rangle \le 1 \quad\; \forall \bx \in S_C \\
& \big\langle M,(\by-\bmu)(\by-\bmu)^{\top} \big\rangle > 1 \quad\; \forall \by \in \SCbar
\end{aligned}
\end{equation}
where, for any two symmetric matrices $A$ and $B$, $\langle A,B \rangle = \trace(AB)$ is the usual Frobenius inner product, implying $\langle M,(\bx-\bmu)(\bx-\bmu)^{\top} \rangle = d_M(\bx,\bmu)^2$.
In words, the constraint $M \succeq \alpha \,\Mstar$ says that $\EL$ must fit into $\MVE$ if we scale $\MVE$ by a factor $\strtch=\nicefrac{1}{\sqrt{\alpha}}$.
The other constraints require $\EL$ to contain all of $S_C$ but none of the points of $\SCbar$.
The objective function thus minimizes the stretch $\strtch$ of $\EL$.

In the rest of this paragraph we prove Theorem~\ref{thm:stretch}.

\paragraph{Proof of Theorem~\ref{thm:stretch} (sketch).}
To build the intuition, we first give a proof sketch where the involved quantities are simplified.
The analysis is performed in the latent space $\R^d$ with inner product $\dotp{\bu,\bv} = \bu^{\top}W\bv$.
Setting conventionally $\bz=\orig$, $C$ then lies in the unit ball $\Bo$ and all points of $X \setminus C$ lie outside $\sqrt{1+\gamma}\,\Bo$.
For simplicity we assume $\gamma \ll 1$ so that $\sqrt{1+\gamma} \simeq 1+\gamma$, but we can easily extend the result to any $\gamma>0$.
Now fix the subset $S_C \subseteq C$, and let $\MVE=\MVE(S_C)$ be the MVEE of $S_C$.
Observe the following fact: $\Bo$ trivially satisfies (1), but in general violates (2); in contrast, $\MVE$ trivially satisfies (2), but in general violates (1).
The key idea is thus to ``compare'' $\Bo$ and $\MVE$ and take, loosely speaking, the best of the two.
To see how this works, suppose for instance $\MVE$ has small radius, say less than $\nicefrac{\gamma}{4}$.
In this case, $\EL=\MVE$ yields the thesis.
Indeed, since the center $\bmustar$ of $\MVE$ is in $\Bo$, then any point of $\EL$ is within distance $1+\nicefrac{\gamma}{4} \le  \sqrt{1+\gamma}$ of the center of $\Bo$, and lies inside $ \sqrt{1+\gamma}\,\Bo$.
Thus $\MVE$ separates $S_C$ from $X \setminus C$, satisfying (1).
At the other extreme, suppose $\MVE$ is large, say with all its $d$ semiaxes longer than $\nicefrac{\gamma}{4}$.
In this case, $\EL=\Bo$ yields the thesis: indeed, by hypothesis $\EL$ fits entirely inside $\nicefrac{4}{\gamma}\,\MVE$, satisfying (2).
Unfortunately, the general case is more complex, since $\MVE$ may be large along some axes and small along others.
In this case, both $\Bo$ and $\MVE$ fail to satisfy the properties.
This requires us to choose the axes and the center of $\EL$ more carefully.
We show how to do this with the help of Figure~\ref{fig:balls}.

\begin{figure}[h!]
\centering
\scalebox{1.1}{\begin{tikzpicture}[scale=.8]
\clip(-3.2,-3.2) rectangle (3.2,3.2);
   \draw[ell,black] (0,0) circle [x radius=3, y radius=3]; 
   \draw[line width=2,red,draw opacity=.1] (2,-4) -- (2,4);
   \draw[line width=2,red,draw opacity=.2] (2,-2.25) -- (2,2.25);
   \node[left] (Bo) at (-2.3,2) {\small $\Bo$};
   \node[dot] (z) at (0,0) {};
   \node[right] (ztxt) at (0,0) {$\bz$};
   \node[right] (U) at (2,2.8) {\small $U$};
   \node[left] (B) at (2.05,1.5) {$B$};
   \draw[ell,red] (2,-1.4) circle [x radius=.2, y radius=1.8]; 
   \node[left] (MVE) at (3,-1) {$\MVE$};
   \node[dot,red] (mustar) at (2,-1.4) {};
   \node[left=2pt] (mustartxt) at (mustar) {$\bmustar$};
\end{tikzpicture}}
\hspace*{4em}
\scalebox{1.1}{\begin{tikzpicture}[scale=.8]
\clip(-3.2,-3.2) rectangle (3.2,3.2);
   \draw[ell,black] (0,0) circle [x radius=3, y radius=3]; 
   \draw[line width=2,red,draw opacity=.1] (2,-4) -- (2,4);
   \draw[line width=2,red,draw opacity=.2] (2,-2.25) -- (2,2.25);
   \node[right] (U) at (2,2.8) {\small $U$};
   \node[left] (E) at (1.7,1) {$E$};
   \draw[ell,blue] (2,0) circle [x radius=.4, y radius=2.5];
   \node[dot] (mu) at (2,0) {};
   \node[left=-1.8pt] (mutxt) at (mu) {$\bmu$};
   \node[dot,red] (mustar) at (2,-1.4) {};
\node[left] (Bo) at (-2.3,2) {\small $\Bo$};
   \node[dot] (z) at (0,0) {};
   \node[right] (ztxt) at (0,0) {$\bz$};
\end{tikzpicture}
   }
\hspace*{2em}
\caption{\small
Left: the MVEE $\MVE$ of $S_C$ and the affine subspace $U+\bmustar$ (marked simply as $U$) spanned by its largest semiaxes.
There is no guarantee that $\MVE \subseteq \sqrt{1+\gamma}\,\Bo$.
Right: the separator $\EL$, centered in the center $\bmu$ of $B$, with the largest semiaxis in $U$ and the smallest one in $U_{\bot}$. We can guarantee that $S_C \subset \EL \subset \sqrt{1+\gamma}\,\Bo$.
}
\label{fig:balls}
\end{figure}
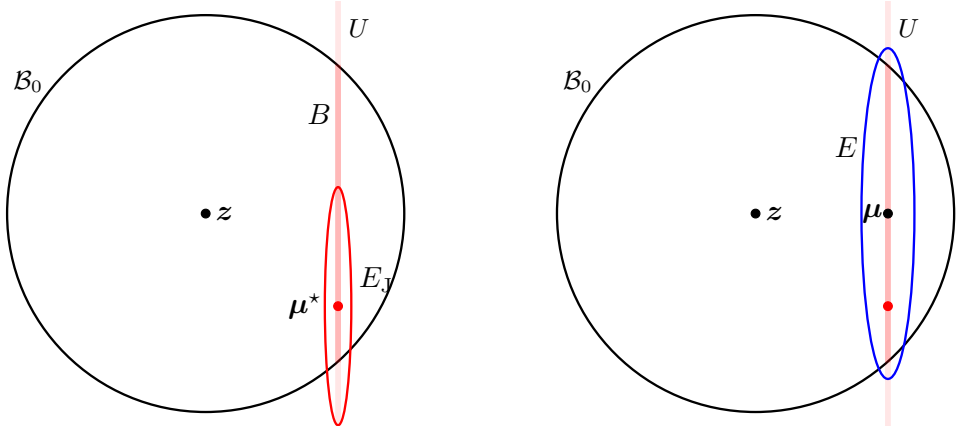
Let $\{\bu_1,\ldots,\bu_d\}$ be the orthonormal basis defined by the semiaxes of $\MVE$ and $\ellstar_1,\ldots,\ellstar_d$ be the corresponding semiaxes lengths.
We define a threshold $\epsilon = \nicefrac{\gamma^3}{d^2}$, and partition $\{\bu_1,\ldots,\bu_d\}$ as
$
A_P = \{i \,:\, \ellstar_i > \epsilon \}
$
and
$
A_Q = \{i \,:\, \ellstar_i \le \epsilon \}
$.
Thus $A_P$ contains the large semiaxes of $\MVE$ and $A_Q$ the small ones.
Let $U, U_{\bot}$ be the subspaces spanned by $\{\bu_i : i \in A_P\}$ and $\{\bu_i : i \in A_Q\}$, respectively.
Consider the subset $B = \Bo \cap (\bmustar + U)$.
Note that $B$ is a ball in at most $d$ dimensions, since it is the intersection of a $d$-dimensional ball and an affine linear subspace of $\R^d$.
Let $\bmu$ and $\ell$ be, respectively, the center and radius of $B$.
We set the center of $\EL$ at $\bmu$, and the lengths $\ell_i$ of its semiaxes as follows:
\begin{align}
\ell_i = \left\{
\begin{array}{cl}
\frac{\ell}{\sqrt{1 - \gamma}} & \text{if } i \in A_P \\[5pt]
\frac{\ellstar_i}{\sqrt{\epsilon}} & \text{if } i \in A_Q
\end{array}
\right.
\end{align}
Loosely speaking, we are ``copying'' the semiaxes from either $\Bo$ or $\MVE$ depending on $\ellstar_i$.
In particular, the large semiaxes (in $A_P$) are set so to contain all of $B$ and exceed it by a little, taking care of not intersecting $\sqrt{1+\gamma}\,\Bo$.
Instead, the small semiaxes (in $A_Q$) are so small that we can safely set them to $1/\sqrt{\ve}$ times those of $\MVE$, so that we add some ``slack'' to include $S_C$ without risking to intersect $\sqrt{1+\gamma}\,\Bo$.
Now we are done, and our low-stretch separator is $(M,\bmu)$ where  $M = \sum_{i=1}^d \!\ell_i^{-2} \bu_i \tp{\bu_i}$.
This the ellipsoid $\EL$ that yields Theorem~\ref{thm:stretch}.
In the next paragraph, we show how we can find efficiently all points in $\EL$ that belong to $C$.

\subsection{Proof of Theorem~\ref{thm:stretch} (full).}
\let\oldMVE\MVE
\renewcommand{\MVE}{E^{\star}}
\label{sub:proof:stretch}
We prove the theorem for $\gamma \le \nicefrac{1}{5}$ and use the fact that whenever $C$ has weak margin $\gamma$ then it also has weak margin $\gamma'$ for all $\gamma' > \gamma$. 
As announced, the analysis is carried out in the latent space $\R^d$ equipped with the inner product $\dotp{\bu,\bv} = \bu^{\top}W\bv$. All norms $\|\bu\|$, distances $d(\bu,\bv)$, and (cosine of) angles $\dotp{\bu,\bv}\big/\big(\|\bu\|\,\|\bv\|\big)$ are computed according to this inner product unless otherwise specified.  
Let $\Bo$ be the unit ball centered at the origin, which we conventionally set at $\bz$, the point in the convex hull of $C$ according to which the margin is computed. Then, by assumption, $C \subset \Bo$, and $\bx \notin \sqrt{1+\gamma}\,\Bo$ for all $\bx \notin C$.
For ease of notation, in this proof be denote the MVEE by $\MVE$ rather than $\oldMVE$.
Let then $(\MVE,\bmustar)$ be the MVEE of $S_C$; note that $\bmustar \in \conv(S_C) \subseteq \Bo$.
We let $\bu_1,\ldots,\bu_d$ be the orthonormal eigenvector basis given by the axes of $\MVE$ and $\lambdastar_1,\ldots,\lambdastar_d$ the corresponding eigenvalues.
Note that if $\min_i \lambdastar_i \ge \nicefrac{5}{\gamma^2}$ then $\MVE$ has radius $\le \nicefrac{\gamma}{\sqrt{5}}$
and thus, since $\bmustar \in \Bo$ and $\gamma\le\nicefrac{1}{5}$, its distance from $\Bo$ is at most $1+\nicefrac{\gamma}{\sqrt{5}} = \sqrt{1+\nicefrac{2\gamma}{\sqrt{5}}+\nicefrac{\gamma^2}{5}} < \sqrt{1+\gamma}$.
In this case we can simply set $\EL=\MVE$ and the thesis is proven.
Thus, from now on we assume $\min_i \lambdastar_i < \nicefrac{5}{\gamma^2}$.
\begin{figure}[h]
\begin{tikzpicture}[scale=.8]
\clip(-3.2,-3.2) rectangle (3.2,3.2);
   \draw[ell,black] (0,0) circle [x radius=3, y radius=3]; 
   \draw[line width=2,red,draw opacity=.1] (2,-4) -- (2,4);
   \draw[line width=2,red,draw opacity=.2] (2,-2.25) -- (2,2.25);
   \node[left] (Bo) at (-2.3,2) {\small $\Bo$};
   \node[dot] (z) at (0,0) {};
   \node[right] (ztxt) at (0,0) {$\bz$};
   \node[right] (U) at (2,2.8) {\small $U$};
   \node[left] (B) at (2.05,1.5) {$B$};
   \draw[ell,red] (2,-1.4) circle [x radius=.2, y radius=1.8]; 
   \node[left] (MVE) at (3,-1) {$\MVE$};
   \node[dot,red] (mustar) at (2,-1.4) {};
   \node[left=2pt] (mustartxt) at (mustar) {$\bmustar$};
\end{tikzpicture}
\hfill
\begin{tikzpicture}[scale=.8]
\clip(-3.2,-3.2) rectangle (3.2,3.2);
   \draw[ell,black] (0,0) circle [x radius=3, y radius=3]; 
   \draw[line width=2,red,draw opacity=.1] (2,-4) -- (2,4);
   \draw[line width=2,red,draw opacity=.2] (2,-2.25) -- (2,2.25);
   \node[right] (U) at (2,2.8) {\small $U$};
   \node[left] (E) at (1.7,1) {$E$};
   \draw[ell,blue] (2,0) circle [x radius=.4, y radius=2.5];
   \node[dot] (mu) at (2,0) {};
   \node[left=-1.8pt] (mutxt) at (mu) {$\bmu$};
   \node[dot,red] (mustar) at (2,-1.4) {};
\node[left] (Bo) at (-2.3,2) {\small $\Bo$};
   \node[dot] (z) at (0,0) {};
   \node[right] (ztxt) at (0,0) {$\bz$};
   \end{tikzpicture}
\hfill
\begin{tikzpicture}[scale=.85]
\clip(4,-0.5) rectangle (7.5,5.5);
   \draw[line width=2,red,draw opacity=.1] (6,-2) -- (6,6);
   \draw[line width=2,red,draw opacity=.2] (6,-2) -- (6,4.35);
   \node[dot] (mu) at (6,0) {};
   \node[right] (mulab) at (mu) {$\bmu$};
   \node[dot] (x) at (5.5,3) {};
   \node[left] (xlab) at (x) {$\bx$};
   \draw[dotted] (x) -- (6, 0 |- x) node {};
   \draw[dotted] (x) -- (x |- 0, 0) node {};
   \draw[black] (x |- 0, 0) -- (6,0) node[midway,below] {$\bq$};
   \draw[black] (6, 0 |- x) -- (mu) node[midway,right] {$\bp$};
   \draw[thick] (-3,0) circle [x radius=10, y radius=10]; 
   \node (orig) at (0,0) {};
   \draw[ell,blue] (mu) circle [x radius=.75, y radius=5]; 
   \node[right] (EL) at (4.6,.7) {$\EL$};
\end{tikzpicture}
\caption{\label{fig:affine_sep}\small Left: the separating ball $\Bo$ of $C$, the MVEE $\MVE$ of $S_C$, and the affine subspace $U+\bmustar$ spanned by its largest semiaxes. Middle: $\EL$ is our separator centered in the center $\bmu$ of the ball $B = U \cap \Bo$. Right: a point $\bx \in S_C$ with its projections onto $U$ and $U_{\bot}$ with respect to the origin, which we conventionally set at $\bmu$ (the center of $\EL$).}
\end{figure}
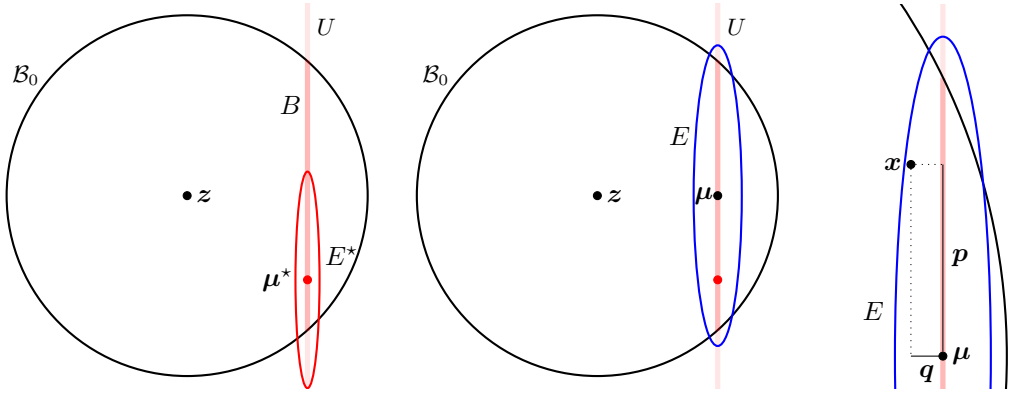

Now let:
\begin{align}
\epsilon = \frac{\gamma^3}{32 d^2}
\end{align}
and partition (the indices of) the basis $\{\bu_1,\ldots,\bu_d\}$ as follows: 
\begin{align}
A_P &= \{i \,:\, \lambdastar_i < \nicefrac{1}{\epsilon^2} \}, \quad A_Q = [d] \setminus A_P
\end{align}
Since $\min_i \lambdastar_i < \nicefrac{5}{\gamma^2}$ and $\nicefrac{5}{\gamma^2} \le \nicefrac{1}{\ve^2}$, then by construction the set $A_P$ is not empty.
We now define the ellipsoid $\EL$.
Let $U, U_{\bot}$ be the subspaces spanned by $\{\bu_i : i \in A_P\}$ and $\{\bu_i : i \in A_Q\}$ respectively, and let $B = \Bo \cap (\bmustar + U)$. Note that $B$ is a ball, since it is the intersection of a ball and an affine linear subspace.
Let $\bmu$ and $\ell$ be, respectively, the center and radius of $B$ and define
\begin{align}
\lambda_i = \left\{
\begin{array}{ll}
(1 - \sqrt{5\gamma/4})\ell^{-2} & i \in A_P \\[5pt]
\epsilon \lambdastar_i & i \in A_Q
\end{array}
\right.
\qquad
M = \sum_{i=1}^d \lambda_i \bu_i \tp{\bu_i} 
\end{align}
Then our ellipsoidal separator is
$
\EL = \{ \bx \in \R^d \,:\, d_M(\bx,\bmu) \le 1 \}
$. See Figure~\ref{fig:affine_sep} for a pictorial representation. 
We now prove that $\EL$ satisfies: \textbf{(1)} $S_C \subset \EL$,\, \textbf{(2)} $\ELOUT \subseteq \frac{64\sqrt{2}d^2}{\gamma^3} \MVE(S_C)$,\, \textbf{(3)} $\EL \subset \sqrt{1+\gamma}\,\Bo$.

\paragraph{Proof of (1).}
Set the center $\bmu$ of $\EL$ as the origin.
For all $i \in [d]$ let $U_i=\bu_i\tp{\bu_i}$ and define the following matrices:
\begin{alignat}{2}
&P_0=\sum_{i \in A_P} \!U_i, &&Q_0=\sum_{i \in A_Q} \!U_i
\\
&P=\sum_{i \in A_P} \!\lambda_i U_i,
&&Q=\sum_{i \in A_Q} \!\lambda_i U_i
\\
&P_{\star}=\sum_{i \in A_P} \!\lambdastar_i U_i,\quad &&Q_{\star}=\sum_{i \in A_Q}\! \lambdastar_i U_i
\end{alignat}
We want to show that $d_M^2(\bx,\bmu) \le 1$ for all $\bx \in S_C$.
Note that $d_M(\bx,\bmu)^2$ equals (recall that $\bmu = \orig$):
\begin{align}
\tp{\bx} P \bx
+ \tp{\bx} Q \bx
\label{eqn:dxmu}
\end{align}
Let us start with the second term of~\eqref{eqn:dxmu}.
By definition of $Q_{\star}$ and since $\tp{\bmustar} Q_{\star}= \tp{(\bmustar-\bmu)} Q_{\star} = \pmb{0}$ because $\bmustar-\bmu \in U$,
\begin{align}
\tp{\bx} Q \bx
=
\epsilon\, \tp{\bx} Q_{\star} \bx
=
\epsilon\, \tp{(\bx-\bmustar)} Q_{\star} (\bx-\bmustar)
\le \epsilon < \frac{\gamma}{4}
\label{eqn:xQx}
\end{align}
where the penultimate inequality follows from $\bx \in \MVE$.

We turn to the first term of~\eqref{eqn:dxmu}.
If we let $\bp,\bq$ be the projections of $\bx-\bmu=\bx$ onto $U,U_{\bot}$, so that
\begin{align}
\norm{}{\bp}^2 = \tp{\bx} P_0 \bx, \qquad \norm{}{\bq}^2 = \tp{\bx} Q_0 \bx
\end{align}
then by definition of the $\lambda_i$ we have:
\begin{align}
\tp{\bx} P \bx
= \frac{1-\sqrt{5\gamma/4}}{\ell^2} \norm{}{\bp}^2
\end{align}
We can thus focus on bounding $\norm{}{\bp}$.
Since $B$ is a ball of radius $\ell$, then $\norm{}{\bp} \le \ell + d(\bp,B)$, where $d(\bp,B)$ is the distance of $\bp$ from its projection on $B$ ---see Figure~\ref{fig:help2}, left.

\begin{figure}[h]
\hfill
\begin{tikzpicture}[scale=5]
\clip(5.5,4.2) rectangle (6.8,5.1);
   \draw[thick,opacity=.7] (-3,0) circle [x radius=10, y radius=10];
   \node[] (B0) at (5.6,5) {$\Bo$};
   \coordinate (BB) at (6,4.36) {};
   \node[] (EL) at (6.23,4.9) {$\EL$};
   \draw[line width=2,red,draw opacity=.1] (6,-2) -- (6,6);
   \draw[line width=2,red,draw opacity=.2] (6,-2) -- (BB);
   \node[dot] (mu) at (6,0) {};
   \node[right] (mulab) at (mu) {$\bmu$};
   \node[dot] (x) at (5.8,4.7) {};
   \node[above right] (xlab) at (x) {$\bx$};
   \node[dot] (p) at (6, 0 |- x) {};
   \node[above right] (plab) at (p) {$\bp$};
   \draw[dotted] (p) -- ($(p)+(0.05,0)$);
   \draw[dotted] (BB) -- ($(BB)+(.05,0)$);
   \draw[<->] ($(BB)+(.05,0)$) -- ($(p)+(0.05,0)$);
   \node[right] (dpb) at ($(BB)+(.04,0.17)$) {$d(\bp,B)$};
   \draw[<->] (x |- 0, 0 |- BB) -- (BB);
   \node[right=3pt] (q) at ($(x |- 0, 0 |- BB)+(0,-.05)$) {$\norm{}{\bq}$};
   \draw[densely dotted] (x) -- (p);
   \draw[densely dotted] (x) -- (x |- 0, 0) node {};
   \draw[black] (x |- 0, 0) -- (6,0);
   \draw[black] (6, 0 |- x) -- (mu) node[right] {$\bp$};
\node (orig) at (0,0) {};
   \draw[ell,blue] (mu) circle [x radius=.75, y radius=5]; 
   \node[right] (EL) at (4.6,.7) {$\EL$};
\end{tikzpicture}
\hfill
\begin{tikzpicture}[scale=1]
\clip(2,-0.4) rectangle (6.5,4);
   \draw[thick,opacity=.7] (0,0) circle [x radius=5, y radius=5];
   \draw[thin,->] (0,0) -> (6,0);
   \coordinate (x0) at (4,3);
   \coordinate (q0) at ($(x0)+(-0.5,0)$);
   \coordinate (p0) at ($(q0)+(0,0.55)$);
   \draw[densely dotted] (x0) -- (q0) -- (p0);
   \node[below right] (q0txt) at (q0) {$a$};
   \node[below left] (p0txt) at (p0) {$b$};
   
   \coordinate (x) at (5,0.02);
   \coordinate (q) at ($(x)+(-0.5,0)$);
   \coordinate (p) at ($(q)+(0,2.2)$);
   \draw[densely dotted] (x) -- (q) -- (p);
   \node[below right] (qtxt) at (q) {$a$};
   \node (ptxt) at (4.3,1) {$b$};
   
\end{tikzpicture}
\hfill
\caption{\label{fig:help2}Left: a point $\bx \in S_C \subset \Bo$ which lies in $\EL$ as well. Right: for a fixed $a>0$, the ratio $\nicefrac{b}{a}$ is maximized when the segment of length $a$ lies on the line passing through the center of $\Bo$, in which case $\nicefrac{b}{a}=\frac{\sin\theta}{1-\cos\theta}$ for some $\theta \in (0,\nicefrac{\pi}{2})$.}
\end{figure}
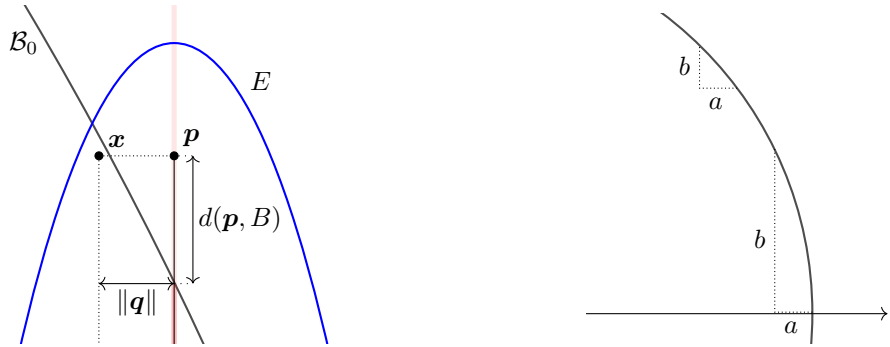

Now, since $\bx \in \Bo$,
the ratio $\frac{d(\bp,B)}{\norm{}{\bq}}$ is maximized when $\ell\to 0$ (i.e., $B$ has a vanishing radius), in which case $d(\bp,B) \le \sin\theta$ and $\norm{}{\bq} \ge 1 - \cos\theta$, where $\theta\in (0,\nicefrac{\pi}{2}]$; see Figure~\ref{fig:help2} right.
Then:
\begin{align}
\frac{\norm{}{\bq}}{d(\bp,B)} \ge \frac{1-\cos\theta}{\sin\theta} = \tan\frac{\theta}{2} \ge \frac{\theta}{2} \ge \frac{\sin\theta}{2} \ge \frac{d(\bp,B)}{2}
\end{align}
where we used the tangent half-angle formula and the Taylor expansion of $\tan\theta$.
This yields $d(\bp,B) \le \sqrt{2 \, \norm{2}{\bq}}$.
Thus:
\begin{align}
\norm{}{\bp} \le \ell + \sqrt{2 \norm{}{\bq} }
\end{align}
But since $\lambdastar_i \ge \nicefrac{1}{\epsilon^2}$ for all $i \in A_Q$:
\begin{align}
\norm{}{\bq}^2 
= \tp{\bx} Q_0 \bx
\le \epsilon^2\, \tp{\bx} Q_{\star} \bx
= \epsilon^2 \tp{(\bx-\bmustar)} Q_{\star} (\bx-\bmustar)
\le \epsilon^2 
\end{align}
Therefore:
\begin{align}
\tp{\bx} P \bx
\le \frac{1-\sqrt{5\gamma/4}}{\ell^2} \big(\ell + \sqrt{2\epsilon}\big)^2
\le (1-\sqrt{5\gamma/4}) \big(1 + \nicefrac{\sqrt{2\epsilon}}{\ell}\big)^2
\label{eqn:xPx}
\end{align}
Next, we show that $\frac{\sqrt{2\epsilon}}{\ell} \le \frac{1}{2}\sqrt{\nicefrac{5\gamma}{4}}$.
First,
\begin{align}
\sqrt{2\epsilon} = \sqrt{2 \frac{\gamma^3}{32 d^2}} = \frac{\gamma \sqrt{\gamma}}{4d}
\label{eqn:sqrt2eps}
\end{align}
We now temporarily set $\bmustar$ as the origin.
We want to show that the projection of $\nicefrac{1}{d}\,\MVE$ on $U$ is contained in $B$.
Now, the projection of an ellipsoid on the subspace spanned by a subset of its axes is a subset of the ellipsoid itself, and $U$ is by definition spanned by a subset of the axes of $\MVE$.
Therefore the projection $P$ of $\nicefrac{1}{d}\,\MVE$ on $U$ satisfies $P \subseteq \nicefrac{1}{d}\,\MVE$.
Suppose then by contradiction that $P \not\subseteq B$.
Since $B=U \cap \Bo$, this implies that $\nicefrac{1}{d}\,\MVE \notin \Bo$.
But by John's theorem, $\nicefrac{1}{d}\,\MVE \subseteq \conv(S_C)$, and therefore $\conv(S_C) \notin \Bo$, which is absurd.
Therefore $P \subseteq B$.

Let us get back to the proof, with $\bmu$ as the origin.
On the one hand, the definitions of $A_P$ and $U$ imply that the largest semiaxis of $\MVE$ of length $\ellstar = 1/\sqrt{\min_i \lambdastar_i}$ lies in $U$, thus $P$ has radius at least $\frac{1}{d}\ellstar$.
On the other hand $B$ has radius $\ell$, and we have seen that $P \subseteq B$.
Therefore, $\ell \ge \frac{1}{d}\ellstar$.
Finally, by our assumption on $\min_i \lambdastar_i$, we have $\min_i \lambdastar_i < \nicefrac{5}{\gamma^2}$ and so $\ellstar > \nicefrac{\gamma}{\sqrt{5}}$.
Therefore, $\ell \ge \nicefrac{\gamma}{\sqrt{5}d}$, which together with~\eqref{eqn:sqrt2eps} guarantees $\frac{\sqrt{2\epsilon}}{\ell} \le \frac{\sqrt{5\gamma}}{4} = \frac{1}{2}\sqrt{\nicefrac{5\gamma}{4}}$.
Thus, continuing~\eqref{eqn:xPx}:
\begin{align}
\tp{\bx} P \bx
&\le (1-\sqrt{5\gamma/4}) \Big(1+\frac{1}{2}\sqrt{\nicefrac{5\gamma}{4}}\Big)^2
\label{eqn:xPx0}
\end{align}
Now $(1-x)(1+\frac{x}{2})^2 < 1 - \frac{3}{4}x^2$ for all $x>0$, thus with $x=\sqrt{5\gamma/4}>\sqrt{\gamma}$ we get:
\begin{align}
\tp{\bx} P \bx
&< 1-\frac{3}{4}\gamma
\label{eqn:xPx1}
\end{align}
By summing~\eqref{eqn:xQx} and~\eqref{eqn:xPx1}, we get:
\begin{align}
\tp{\bx} P \bx + \tp{\bx} Q \bx < 1 - \frac{3}{4}\gamma + \frac{\gamma}{4} < 1
\end{align}

\paragraph{Proof of (2).}
Comparing the eigenvalues of $\EL$ and $\MVE$, and using $\ell \le 1$ and $\gamma \le \nicefrac{1}{5}$, we obtain:
\begin{align}
\frac{\lambda_i}{\lambdastar_i} \ge \left\{
\begin{array}{ll}
\frac{\nicefrac{(1-\sqrt{5\gamma/4})}{\ell^2}}{\nicefrac{1}{\epsilon^2}} \ge \frac{\epsilon^2}{2}
& \; i \in A_P
\\
\epsilon > \frac{\epsilon^2}{2}
& \; i \in A_Q
\end{array}
\right.
\end{align}
Thus the semiaxes lengths of $\EL$ are at most $\nicefrac{\sqrt{2}}{\epsilon}$ times those of $\MVE$.
Now let $\MVE_+$ be the set obtained by scaling $\MVE$ by a factor $\nicefrac{2\sqrt{2}}{\epsilon}=\nicefrac{64\sqrt{2}d^2}{\gamma^3}$ about its origin $\bmustar$.
Note that $\bmustar \in \conv(S_C)$ and, by item \textbf{(1)}, $\conv(S_C) \subseteq \EL$, which implies $\bmustar\in\EL$.
Now, $\MVE_+$ contains any set of the form $\by+\frac{1}{2}\MVE_+$ if the latter contains $\bmustar$; this includes the set $\frac{\sqrt{2}}{\epsilon} \MVE$ centered in $\bmu$, which in turn contains $\EL$ as we already said.

\paragraph{Proof of (3).}
We prove that $d(\bx,\Bo)^2 < \gamma$ for all $\bx \in \EL$.
Since $\Bo$ is the unit ball, this implies $\EL \subset \sqrt{1+\gamma}\,\Bo$.
Consider then any such $\bx$.
Let again $\bp,\bq$ be the projections of $\bx$ on $U$ and $U_{\bot}$ respectively.
Because $B \subseteq \Bo$, $d(\bx,\Bo)^2 \le d(\bx,B)^2 = d(\bp,B)^2 + \norm{}{\bq}^2$.
See again Figure~\ref{fig:help2}, left, but with $\bx$ possibly outside $\Bo$.
For the first term, note that
\begin{align}
d(\bp,B) \le \max_{i \in A_P} \sqrt{\nicefrac{1}{\lambda_i}} \, - \ell
\end{align}
By definition of $\lambda_i$, this yields:
\begin{align}
d(\bp,B)^2
& \le \left(\frac{\ell}{\sqrt{1-\sqrt{5\gamma/4}}} - \ell\right)^2
 \le \left(\frac{1}{\sqrt{1-\sqrt{5\gamma/4}}} - 1\right)^2
 \tag{because $\ell \le 1$}
\end{align}
Now we show that the right-hand side is bounded by $\frac{3}{4}\gamma$.
Consider $f(x)=\frac{1}{\sqrt{1-x}}-1$ for $x \in [0,\nicefrac{1}{2}]$.
Now $\frac{\partial^2 f}{\partial x^2} = \frac{3}{4}(1-x)^{-\nicefrac{5}{2}}>0$, so $f$ is convex.
Moreover, $f(\nicefrac{1}{2}) = \sqrt{2}-1 < 0.83 \cdot \nicefrac{1}{2}$, and clearly $f(0) = 0 \le 0.83 \cdot 0$.
By convexity then, for all $x \in [0,\nicefrac{1}{2}]$ we have $f(x) \le 0.83\, x$ which implies $f(x)^2 < 0.75\, x^2$.
By substituting $x = \sqrt{5\gamma/4}$, for all $\gamma \le \nicefrac{1}{5}$ we obtain:
\begin{align}
d(\bp,B)^2 \le \left(\frac{1}{\sqrt{1-\sqrt{5\gamma/4}}} - 1\right)^2 < \frac{3}{4} \cdot \frac{5}{4} \gamma = \frac{15}{16} \gamma
\label{eqn:dpB}
\end{align}

Let us now turn to $\bq$.
By definition of $Q_0$, of $Q$, and of $\lambda_i$ for $i \in A_Q$, we have:
\begin{align}
\norm{}{\bq}^2
= \tp{\bx} Q_0 \bx
\le \max_{i \in A_Q} \frac{1}{\lambda_i} \tp{\bx} Q \bx
= \max_{i \in A_Q} \frac{1}{\epsilon \lambdastar_i} \tp{\bx} Q \bx
\end{align}
But $\tp{\bx} Q \bx \le 1$ since $\bx \in E$, and recalling that $\lambdastar_i \ge \nicefrac{1}{\epsilon^2}$ for all $i \in A_Q$, we obtain:
\begin{align}
\norm{}{\bq}^2
\le \frac{1}{\epsilon (\nicefrac{1}{\epsilon^2})}
= \epsilon < \frac{\gamma}{16}
\label{eqn:normq}
\end{align}
Finally, by summing~\eqref{eqn:dpB} and~\eqref{eqn:normq}:
\begin{align}
d(\bx,\Bo)^2 \le d(\bp,B)^2 + \norm{}{\bq}^2 < \gamma
\end{align}
The proof is complete.
\renewcommand{\MVE}{\oldMVE}

\section{Supplementary material for Section~\ref{sec:all}}
\label{apx:all}

\subsection{Lemma~\ref{lem:draw}}
\begin{lemma}
\label{lem:draw}
Let $b > 0$ be a sufficiently large constant.
Let $S$ be a sample of points drawn independently and uniformly at random from $X$.
Let $C = \arg \max_{C_j \in \C} |S \cap C_j|$, let $S_C = S \cap C$, and suppose $|S_C| \ge b d^2 \ln k$.
If $E$ is any (possibly degenerate) ellipsoid in $\R^d$ such that $S_C = C \cap E$, then with probability at least $\nicefrac{1}{2}$ we have $|C \cap \EL| \ge |X|\frac{1}{4k}$.
The same holds if we require that $\EL \cap (S \setminus S_C) = \emptyset$, i.e., that $\EL$ separates $S_C$ from $S \setminus S_C$.
\end{lemma}
\begin{proof}
Let $n=|X|$ for short, and for any ellipsoid $\EL$ let $\EL_X=\EL\cap X$.
We show that, with $C$ defined as above, \textbf{(i)} with probability at least $1-\nicefrac{1}{4}$ we have $|C| \ge \nicefrac{n}{2k}$, and \textbf{(ii)} with probability at least $1-\nicefrac{1}{4}$, if $|C| \ge \nicefrac{n}{2k}$ then $|E_X \SymDif C| \le \nicefrac{1}{2} |C|$ where $\SymDif$ denotes symmetric difference.
By a union bound, then, with probability at least $\nicefrac{1}{2}$ we have $|\EL \cap C| \ge |C| - |\EL_X \SymDif C| \ge \frac{1}{2}|C| \ge \nicefrac{n}{4k}$.

\textbf{(i)}.
Let $S$ be the multiset of samples drawn from $X$, and for every cluster $C_i \in \C$ let $N_i$ be the number of samples in $C_i$.
Let $s=kb d^2 \ln k$; note that $|S| \le s$ since there are at most $k$ clusters.
Now fix any $C_i$ with $|C_i| < \frac{n}{2k}$.
Then $\E[N_{i}] \le s\frac{|C_i|}{n} < \frac{b d^2 \ln k}{2}$, and
by standard concentration bounds (Lemma~\ref{lem:chernoff} in this supplementary material), we have $\Pr(N_{i} \ge b d^2 \ln k) = \exp(-\Omega(b\ln k))$, which for $b$ large enough drops below $\nicefrac{1}{4k}$.
Therefore, the probability that $N_i \ge b d^2 \ln k$ when taking $s \le k b d^2 \ln k$ samples is at most $\nicefrac{1}{4k}$.
By a union bound on all $C_i$ with $|C_i| < \nicefrac{n}{2k}$, then, $|C|\ge \nicefrac{n}{2k}$ with probability $1-\nicefrac{1}{4}$.

\textbf{(ii)}.
Consider now any $C_i$ with $|C_i| \ge \nicefrac{n}{2k}$.
We invoke the generalization bounds of Theorem~\ref{thm:vc+pac} in this supplementary material with $\epsilon=\nicefrac{1}{4k}$ and $\delta=\nicefrac{1}{4k}$, on the hypothesis class $\Hs$ of all (possibly degenerate) ellipsoids in $\R^d$.
For $b$ large enough, the generalization error of any ellipsoid $E$ that contains $S_C$ is, with probability at least $1-\nicefrac{1}{4k}$, at most $\nicefrac{1}{4k}$, which means $|E_X \SymDif C_i| \le \nicefrac{n}{4k} \le \nicefrac{1}{2} |C_i|$, as desired.
By a union bound on all clusters, with probability at least $1-\nicefrac{1}{4}$ this holds for all $C_i$ with $|C_i| \ge \nicefrac{n}{2k}$.
The same argument holds if we require $E$ to separate $S \cap C_i$ from $S \setminus C_i$, see again Theorem~\ref{thm:vc+pac}.
By a union bound with point (i) above, we have $E \cap C \le \nicefrac{1}{2} |C|$ with probability at least $\nicefrac{1}{2}$, as claimed.
\end{proof}

\subsection{Proof of Lemma~\ref{lem:hp_rounds}}
\newcommand{\F}{\mathcal{F}}
Let $X_0=X$ and $N_0=n$, and for all $i \ge 1$, let $X_i$ be the set of points not yet labeled at the end of round $i$, let $N_i=|X_i|$, and let $R_i= \Ind{N_i \le N_{i-1}(1-\nicefrac{1}{4k})}$.
Recall that $S_C$ is large enough so that, by Lemma~\ref{lem:draw} in this supplementary material, we have $\Pr(R_i = 1 \,|\, X_{i-1}) \ge \nicefrac{1}{2}$ for all $i$.
For every $t \ge 1$ let $\rho_t=\sum_{i=1}^t R_i$.
Note that:
\begin{align}
\label{eqn:expround}
N_t &\le N_0  (1-\nicefrac{1}{4k})^{\rho_t} < n e^{-\frac{\rho_t}{4k}}
\end{align}
If $\rho_t \ge 4k \ln(1/\epsilon)$, then $N_t < \epsilon n$ and  \AlgoFULL($X,k,\gamma,\epsilon$) stops.
The number of rounds executed by \AlgoFULL($X,k,\gamma,\epsilon$) is thus at most $r_{\epsilon}=\min\{ t : \rho_t \ge 4k \ln(1/\epsilon)\}$.

Now, for all $i \ge 1$ consider the $\sigma$-algebra $\F_{i-1}$ generated by $X_0,\ldots,X_{i-1}$, and define:
$
Z_i = R_i \, B_i
$,
where $B_1,B_2,\ldots$ are Bernoulli random variables where each $B_i$ has parameter ${1}\big/\big({2 \,\E[R_i \mid \F_{i-1}]}\big)$.
Obviously, $Z_i \le R_i$, and thus for all $t$ we deterministically have:
\begin{align}
\rho_t = \sum_{i=1}^t R_i \ge \sum_{i=1}^t Z_i
\end{align}
Now note that:
\begin{align}
\E[Z_i \mid \F_{i-1}]= \E[R_i \mid \F_{i-1}] \frac{1}{2 \,\E[R_i \mid \F_{i-1}]} = \frac{1}{2}
\end{align}

Now we can prove the theorem.
For the first claim, simply note that $\E[r_{\epsilon}] \le 8k \ln(1/\epsilon)$, as this is the expected number of fair coin tosses to get $4k \ln(1/\epsilon)$ heads.

For the second claim, consider any $t \ge 8k \ln n + 6 a \sqrt{k} \ln n$.
Letting $\zeta_t=\sum_{i=1}^t Z_t$, the event $r_{0} \ge t$ implies $\zeta_t < 4k \ln n = \frac{t}{2}-3a\sqrt{k}\ln n = \E[\zeta_t]-\delta$ where $\delta=3 a \sqrt{k}\ln n$.
By Hoeffding's inequality this event has probability at most $e^{-2\delta^2/t}$, and one can check that for all $a \ge 1$ we have $\frac{2\delta^2}{t} \ge a \ln n$.

\section{Supplementary material for Section~\ref{sec:lb}}
\label{apx:lb}

\subsection{Proof of Theorem~\ref{thm:LB}}
We state and prove two distinct theorems which immediately imply Theorem~\ref{thm:LB}.

\begin{theorem}
\label{thm:LB1}
For all $0 < \gamma < \nicefrac{1}{7}$, all $d \ge 2$, and every (possibly randomized) learning algorithm, there exists an instance on $n \ge 2(\frac{1+\gamma}{8\gamma})^{\frac{d-1}{2}} $ points and $|\C|=3$ latent clusters such that (1) all clusters have margin $\gamma$, and (2) to return with probability $\nicefrac{2}{3}$ a clustering $\hat{\C}$ such that $\ErrClust(\hat{\C},\C)=0$ the algorithm must make $\Omega(n)$ same-cluster queries in expectation.
\end{theorem}
\begin{proof}
The idea is the following.
We define a single set of points $X \subset \R^d$ and randomize over the choice of the latent PSD matrix $W$; the claim of the theorem follows by applying Yao's minimax principle.
Specifically, we let $X$ be a $\Theta(\sqrt{\gamma})$-packing of points on the unit sphere in $\R^d$.
We show that, for $\bx \in X$ drawn uniformly at random, setting $W=(1+\gamma)\diag(x_1^2,\ldots,x_d^2)$ makes $\bx$ an outlier, as its distance $d_W(\bx,\orig)$ from the origin is $1+\gamma$, while every other point is at distance $\le 1$. Since there are roughly $(\nicefrac{1}{\gamma})^d$ such points $\bx$ in our set, the bound follows.

We start by defining the points $X$ in terms of their entry-wise squared vectors.
Consider $S_d^+ = \R^d_+ \cap S_d$ where $S_d=\{\bx \in \R^d \,:\; \norm{2}{\bx}=1\}$ is the unit sphere in $\R^d$.
We want to show that there exists a set of $\frac{1}{2}(\nicefrac{1}{\epsilon})^{d-1}$ points in $S_d^+$ whose pairwise distance is bigger than $\nicefrac{\epsilon}{2}$, where $\epsilon$ will be defined later.
To see this, recall that the packing number of the unit ball $B_d=\{\bx \in \R^d\,:\, \norm{2}{\bx} \le 1\}$ is
$
\mathcal{M}(B,\epsilon) \ge (\nicefrac{1}{\epsilon})^d
$ ---see, e.g., \cite{HDPbook}.
For $\nicefrac{\epsilon}{2}$ and $d-1$, this implies there exists $Y \subseteq B_{d-1}$ such that $|Y|\ge (\nicefrac{2}{\epsilon})^{d-1}$ and $\norm{2}{\by-\by'} > \nicefrac{\epsilon}{2}$ for all distinct $\by,\by' \in Y$.
Now, consider the lifting function $f : B_{d-1} \rightarrow \R^{d}$ defined by $f(\by)=(\sqrt{1-\norm{2}{\by}^2}, y_{1},\ldots,y_{d-1})$.
Define the lifted set $Z=\{ f(\by) : \by \in Y\}$.
Clearly, every $\bz \in Z$ satisfies $\norm{2}{\bz}=1$ and $z_0 \ge 0$, so $\bz$ lies on the northern hemisphere of the sphere $S_d$.
Moreover, $\norm{2}{f(\by) - f(\by')} \ge \norm{2}{\by - \by'}$ for any two $\by,\by' \in Y$.
Hence, we have a set $Z$ of $(\nicefrac{2}{\epsilon})^{d-1}$ points on the $d$-dimensional sphere such that $\norm{2}{\bz-\bz'} > \nicefrac{\epsilon}{2}$ for all distinct $\bz,\bz' \in Z$.
But a hemisphere is the union of $2^{d-1}$ orthants, hence some orthant contains at least $2^{-(d-1)}(\nicefrac{2}{\epsilon})^{d-1} = (\nicefrac{1}{\epsilon})^{d-1}$ of the points of $Z$.
Without loss of generality we may assume this is the positive orthant and denote the set as $Z^+$.

We now define the input set $X \subseteq \R^d$ as follows:
\begin{align*}
X = X^+ \cup X^- = \{ \sqrt{\bz} \,:\, \bz \in Z^+ \} \cup \{ -\sqrt{\bz} \,:\, \bz \in Z^+ \}
\end{align*}
Note that $n=|X|=2|Z^+|=2(\nicefrac{1}{\epsilon})^{d-1}$.
Next, we show how every $\bz \in Z^+$ defines a clustering instance satisfying the constraints of the thesis.
For any $\bz^* \in Z^+$; let $\bw=(1+\gamma)\bz^*$ and $W=\diag(w_1,\ldots,w_d)$, which is PSD as required.
Define the following three clusters:
\begin{align*}
    C'=\{-\sqrt{\bz^*}\}
\qquad
    C''=\{\sqrt{\bz^*}\}
\qquad
    C=X \setminus (C' \cup C'')
\end{align*}
where, for $f:\R\to\R$, $f(\bx) = \big(f(x_1),\ldots,f(x_d)\big)$. 
Since $C'$ and $C''$ are singletons, they trivially have weak margin $\gamma$.
We now show that $C$ has weak margin $\gamma$ w.r.t.\ to $\bmu=\pmb{0}$; that is, $d_{W}(\bx,\bmu)^2 > 1+\gamma$ for $\bx=\pm \sqrt{\bz^*}$ and $d_{W}(\bx,\bmu)^2 \le 1$ otherwise.
First, note that $d_{W}(\bx,\bmu)^2 = \dotp{\bw,\bx^2}$.
Now, 
\begin{equation}
\label{eqn:dbxbmu}
    d_{W}(\bx,\bmu)^2
=
    \left\{ \begin{array}{cl}
        (1+\gamma)\dotp{\bz^*,\bz^*} = 1+\gamma & \text{if $\bx \in C',C''$}
    \\
        (1+\gamma)\dotp{\bz^*,\bx^2} & \text{if $\bx \in C$}
    \end{array} \right.
\end{equation}
However, by construction of $Z^+$, we have that for all $\bx \in C$ and $\bz = \bx^2$,
\begin{align*}
(\nicefrac{\epsilon}{2})^2 \le \norm{2}{\bz-\bz^*}^2 = \norm{2}{\bz}^2 - 2\dotp{\bz,\bz^*} + \norm{2}{\bz^*}^2 = 2(1-\dotp{\bz,\bz^*})
\end{align*}
which implies
$
\dotp{\bz^*,\bx^2} \le 1 - (\nicefrac{\epsilon}{2})^2/2 = 1 - \nicefrac{\epsilon^2}{8} = \nicefrac{1}{(1+\gamma)}
$
for $\epsilon=\sqrt{\nicefrac{8\gamma}{(1+\gamma)}}$. Therefore~\eqref{eqn:dbxbmu} gives $d_{W}(\bx,\bmu)^2 = (1+\gamma)\dotp{\bz^*,\bx^2} \le 1$. This proves $C$ has weak margin $\gamma$ as desired.

The size of $X$ is:
\begin{align*}
n \ge 2\Big(\frac{1}{\sqrt{\nicefrac{8\gamma}{(1+\gamma)}}}\Big)^{d-1} 
= 2\Big(\frac{1+\gamma}{8\gamma}\Big)^{\frac{d-1}{2}}  
\end{align*}
Now the distribution of the instances is defined by taking $\bz^*$ from the uniform distribution over $Z^+$.
Consider any deterministic algorithm running over such a distribution.
Note that same-cluster queries always return $+1$ unless at least one of the two queried points is not in $C$.
As $C$ contains all points in $X$ but the symmetric pair $\sqrt{\bz^*}, -\sqrt{\bz^*}$ for a randomly drawn $\bz^*$, a constant fraction of the points in $X$ must be queried before one element of the pair is found with probability bounded away from zero.
Thus, any deterministic algorithm that returns a zero-error clustering with probability at least $\delta$ for any constant $\delta > 0$ must perform $\Omega(n)$ queries.
By Yao's principle for Monte Carlo algorithms then (see Section~\ref{sub:yao} above), any randomized algorithm that errs with probability at most $\frac{1-\delta}{2} \le \frac{1}{2}$ for any constant $\delta > 0$ must make $\Omega(n)$ queries as well.
\end{proof}

\begin{theorem}
\label{thm:LB2}
For all $\gamma > 0$, all $d \ge 48(1+\gamma)^2$, and every (possibly randomized) learning algorithm, there exists an instance on $n = \Omega\big(\exp(d/(1+\gamma)^2)\big)$ points and $|\C|=2$ latent clusters such that (1) all clusters have margin at least $\gamma$, and (2) to return with probability $\nicefrac{2}{3}$ a clustering $\hat{\C}$ such that $\ErrClust(\hat{\C},\C)=0$ the algorithm must make  $\Omega(n)$ same-cluster queries in expectation.
\end{theorem}
\begin{proof}
We exhibit a distribution of instances that gives a lower bound for every algorithm, and then use Yao's minimax principle. 
Let $p=\frac{1}{2(1+\gamma)}$.
Consider a set of vectors $\bx_1,\ldots,\bx_n$ where every entry of each vector $x_{j,i}$ is i.i.d.\ and it is equal to $1$ with probability. $p$.
Define $X=\{\bx_1,\ldots,\bx_n\}$; note that in general $|X| \le n$ since the points might not be all distinct.
Let $\bx^{\star}=\bx_n$, $C=\{\bx_1,\ldots,\bx_{n-1}\}$, $C'=\{\bx^{\star}\}$. The latent clustering is $\C=\{C,C'\}$, and the matrix and center of $C$ are respectively $W=\diag(\bx^{\star})$ and $\bc=\orig$.
The algorithms receive in input a random permutation of $X$; clearly, if it makes $o(|X|)$ queries, then it has vanishing probability to find $\bx^{\star}$, which is necessary to return the latent clustering $\C$.

Now we claim that, if $d \ge {48(1+\gamma)^2}$, then we can set $n = \Omega\Big( \exp\big(\frac{d}{48(1+\gamma)^2}\big)\Big)$ and with constant probability we will have \textbf{(i)} $|X| = \Omega(n)$, and \textbf{(ii)} $C,C'$ have margin $\gamma$.
This is sufficient, since the theorem then follows by applying Yao's minimax principle.

Let us first bound the probability that $|X|<n$.
Note that for any two points $\bx_i,\bx_{i'}$ with $i \ne i'$ we have $\Pr(\bx_i=\bx_{i'}) = ((1-p)^2+p^2)^d<(1-\frac{1}{2(1+\gamma)})^d < e^{-\frac{d}{2(1+\gamma)}}$.
Therefore, by a simple union bound over all pairs, $\Pr(|X|<n) < n^2 e^{-\frac{d}{2(1+\gamma)}}$.

Next, we want show that, loosely speaking, $d_W(\bx,\bc)^2 \simeq dp$ for $\bx \in C'$ whereas $d_W(\bx,\bc)^2 \simeq dp^2$ for $\bx \in C$; this will give the margin.

Now, for any $\bx$,
\begin{align}
d_W(\bx,\bc)^2 = \sum_{i=1}^d x^{\star}_i \, (x_{i}-0)^2 = \left\{
\begin{array}{ll}
\sum_{i=1}^d x^{\star}_i\,x_i \sim B(d,p^2) & \bx \in C  \\[2pt]
\sum_{i=1}^d x^{\star} \sim B(d,p) & \bx \in C'
\end{array}
\right.
\end{align}
Where in the last equality we use the fact that the entries are unary, and where with the notation $B(d,p)$ we refer to a vector of length $d$ where each entry is equal to $1$ with probability $p$.
Let $\mu=dp^2$ and $\mu'=dp$, let $\epsilon=\nicefrac{1}{(1+\sqrt{2})}$, and define
\begin{align}
\phi &= \mu(1+\epsilon),\qquad \phi' = \mu'(1-\epsilon\sqrt{p}) 
\end{align}
By standard tail bounds,
\begin{align}
&\Pr(d_{W}(\bx,\bc)^2 \ge \phi) \le e^{-\frac{\epsilon^2 \mu}{3}} \quad \text{for } \bx \in C
\\
&\Pr(d_{W}(\bx,\bc)^2 < \phi' ) < e^{-\frac{\epsilon^2 p \mu'}{3}} = e^{-\frac{\epsilon^2 \mu}{3}} \quad \text{for } \bx \in C'
\end{align}
By a union bound on all points, the margin $\gamma_C$ of $C$ fails to satisfy the following inequality with probability at most $|X| e^{-\frac{\epsilon^2 \mu}{3}} \le n e^{-\frac{\epsilon^2 \mu}{3}}$:
\begin{align}
1+\gamma_C
= \frac{\min_{\bx \notin C}d_{W}(\bx,\bc)^2}{\max_{\bx \in C} d_{W}(\bx,\bc)^2}
\ge \frac{\phi'}{\phi}
= \frac{dp(1-\epsilon\sqrt{p})}{dp^2(1+\epsilon)}
= \frac{1-\epsilon\sqrt{p}}{p(1+\epsilon)}
\ge \frac{1}{2p} = 1+\gamma
\end{align}
where the penultime inequality holds since $\frac{1-\epsilon\sqrt{p}}{1+\epsilon} \ge \frac{1}{2}$ for our values of $p$ and $\epsilon$.
Note that, since $p = \frac{1}{2(1+\gamma)}$ and $n \le \frac{1}{c}\exp\big(\frac{d}{48(1+\gamma)^2}\big) + 1$,
\begin{align}
n e^{-\frac{\epsilon^2 \mu}{3}} = n e^{-\frac{dp^2}{12}} = n e^{-\frac{d}{48(1+\gamma)^2}}
\end{align}

By one last union bound, the probability that $|X|=n$ and $\gamma_C \ge \gamma$ is at least
\begin{align}
    1 - n e^{-\frac{d}{48(1+\gamma)^2}} - n^2 e^{-\frac{d}{2(1+\gamma)}}
\end{align}
If $d \ge \frac{48}{(1+\gamma)^2}$, then we can let $n=\Omega\Big(e^{\frac{d}{48(1+\gamma)^2}}\Big)$ while ensuring the above probability is bounded away from $0$.

The rest of the proof and the application of Yao's principle is essentially identical to the proof of Theorem~\ref{thm:LB1} above.
\end{proof}

\section{Comparison with \scq-$k$-means}
\label{app:ash}
In this section we compare our algorithm to \scq-$k$-means of \cite{ashtiani2016clustering}.
We show that, in our setting, \scq-$k$-means fails even on very simple instances, although it can still work under (restrictive) assumptions on $\gamma$, $W$, and the centers.

\scq-$k$-means works as follows.
First, the center of mass $\bmu_C$ of some cluster $C$ is estimated using $\scO\big(\!\poly(k,\nicefrac{1}{\gamma})\big)$ \scq\ queries; second, all points in $X$ are sorted by their distance from $\bmu_C$ and the radius of $C$ is found via binary search.
The binary search is done using same-cluster queries between the sorted points and any point already known to be in $C$. 
The margin condition ensures that, if we have an accurate enough estimate of $\bmu_C$, then the binary search will be successful (there are no inversions of the sorted points w.r.t.\ their cluster). This approach thus yields a $\scO(\ln n)$ \scq\ queries bound (the number of queries to estimate $\bmu_C$ is independent of $n$).

It is easy to see that this algorithm relies crucially on (1) each cluster $C$ must be spherical, and (2) the center of the sphere must coincide with the centroid $\bmu_C$.
In formal terms, the setting of~\cite{ashtiani2016clustering} is a special cases of ours where for all $C$ we have $W_C=I_d$ and $\bc=\E_{\bx \in C}[\bx]$.
If any of these two assumptions does not hold, then it is easy to construct instances where~\cite{ashtiani2016clustering} fails to recover the clusters and, in fact, achieves error very close to a completely random labeling.
Formally:
\begin{lemma}
For any fixed $d \ge 2$, any $p \in (0,1)$, and any sufficiently small $\gamma > 0$, there are arbitrarily large instances on $n$ points and $k=2$ clusters on which \scq-$k$-means incurs error $\ErrClust(\hat\C,\C) \ge \frac{1-p}{2}$ with probability at least $1-p$.
\end{lemma}
\begin{proof}[Sketch of the proof]
We describe the generic instance on $n$ points for $d=2$.
The latent clustering $\C$ is formed by two clusters $C_1,C_2$ of size respectively $n_1=n\frac{1+p}{2}$ and $n_2=n\frac{1-p}{2}$.
In $C_1$, half of the points are in $(1,0)$ and half in $(-1,0)$.
In $C_2$, all points are in $(0,\frac{\sqrt{1+\gamma}}{2})$.
(One can in fact perturb the instance so that all points are distinct without impairing the proof).
For both clusters, the center coincide with their center of mass, $\bmu_1=(0,0)$ and $\bmu_2=(0,\frac{\sqrt{1+\gamma}}{2})$.
For both clusters, the latent metric is given by the PSD matrix $W=(\begin{smallmatrix} .25 & 0 \\ 0 & 1
\end{smallmatrix})$.
It is easy to see that $d_W(\bx,\bmu_1)^2=\nicefrac{1}{4}$ if $\bx \in C_1$ and $d_W(\bx,\bmu_1)^2=\nicefrac{(1+\gamma)}{4}$ if $\bx \in C_2$, and so $C_1$ has margin exactly $\gamma$.
On the other hand $C_2$ has margin $\gamma$ since $d_W(\bx,\bmu_2)^2=0$ if $\bx \in C_2$ and $d_W(\bx,\bmu_2)^2>0$ otherwise.
\begin{figure}[h!]
    \centering
    \begin{tikzpicture}[scale=2]
    \draw[black,thin,->] (-2,0) -- (2,0);
    \draw[black,thin,->] (0,-.2) -- (0,1);
    \node[dot,fill=blue,minimum size=4pt] (c11) at (-1,0) {};
    \node[dot,fill=blue,minimum size=4pt] (c12) at (1,0) {};
    \node[dot,fill=green,minimum size=5pt] (c2) at (0,.5) {};
    \draw[densely dotted,opacity=.6] (0,0) circle [x radius=1.2, y radius=.25]; 
    \draw[densely dotted,opacity=.6] (0,.5) circle [x radius=.3, y radius=.0625]; 
    \node (C1) at (.8,.3) {\small $C_1$};
    \node (C2) at (.25,.65) {\small $C_2$};
    \end{tikzpicture}
    \caption{\small A bad instance for \scq-$k$-means. With good probability, the algorithm classifies all points in a single cluster, incurring error $\simeq \nicefrac{1}{2}$, the same as a random labeling.}
    \label{fig:badAshtiani}
\end{figure}
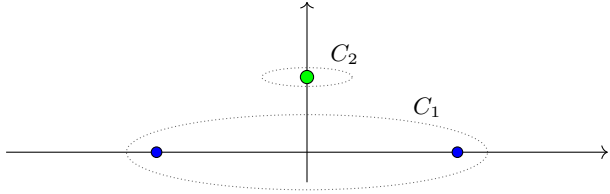

Now consider \scq-$k$-means.
The algorithm starts by sampling at least $\frac{k \ln(k)}{\gamma^4}$ points from $X$ and setting $\hat\bmu$ to the average of the points with the majority label.
By standard concentration bounds then, for $\gamma$ small enough, with probability at least $1-p$ the majority cluster will be $C_1$ and the estimate $\hat\bmu$ of its center of mass $(0,0)$ will be sufficiently close to $\bmu_1$ that the ordering of all points in $X$ by their Euclidean distance w.r.t.\ $\hat\bmu$ will set all of $C_2$ before all of $C_1$.
But since $n_2 = n\frac{1-p}{2}$, the median of the sorted sequence will be a point of $C_1$.
Thus the binary search will make its first query on a point of $C_1$ and will continue thereafter classifying all of $X$ as belonging to $C_1$.
Thus the algorithm will output the clustering $\hat\C = \{X,\emptyset\}$ which gives $\ErrClust(\hat\C,\C)=\frac{1-p}{2}$.
\end{proof}

Next, we show that the approach~\cite{ashtiani2016clustering} still works if one relaxes the assumption $W=I$, at the price of strengthening the margin $\gamma$.
Let $\lambda_{\max}$ and $\lambda_{\min} > 0$ be, respectively, the largest and smallest eigenvalues of $W$.
The condition number $\kappa_W$ of $W$ is the ratio $\lambda_{\max} \big/ \lambda_{\min}$.
If $\kappa_W$ is not too large, then $W$ does not significantly alter the Euclidean metric, and the ordering of the points is preserved.
Formally:
\begin{lemma}
\label{lem:kappa}
Let $\kappa_W$ be the condition number of $W$.
If every cluster $C$ has margin at least $\kappa_W-1$ with respect to its center of mass $\bmu_C$, and if we know $\bmu_C$, then we can recover $C$ with $\scO(\ln n)$ \scq\ queries.
\end{lemma}
\begin{proof}
Fix any cluster $C$ and let $\bmu=\bmu_C$. For any $\bz \in \R^d$ we have
$\lambda_{\min} \norm{2}{\bz}^2 \le \norm{W}{\bz}^2  \le \lambda_{\max} \norm{2}{\bz}^2
$ where $\lambda_{\min}$ and $\lambda_{\max}$ are, respectively, the smallest and largest eigenvalue of $W$.
Sort all other points $\bx$ by their Euclidean distance $\norm{2}{\bx-\bmu}$ from $\bmu$.
Then, for any $\bx \in C$ and any $\by \notin C$ we have:
\begin{align}
\frac{\norm{2}{\by-\bmu}^2}{\norm{2}{\bx-\bmu}^2}
\ge \frac{\lambda_{\min}}{\lambda_{\max}} \frac{\norm{W}{\by-\bmu}^2}{\norm{W}{\bx-\bmu}^2}
= \frac{1}{\kappa_W} \frac{d(\by,\bmu)^2}{d(\bx,\bmu)^2}
> \frac{1+\gamma}{\kappa_W}
\label{eqn:kappa_gamma}
\end{align}
Hence, if $\gamma \ge \kappa_W-1$, there is $r \ge 0$ such that $\norm{2}{\bx-\bmu} \le r$ for all $\bx \in C$ and $\norm{2}{\by-\bmu} \ge r$ all $\by \notin C$.
We can thus recover $C$ via binary search as in~\cite{ashtiani2016clustering}.
\end{proof}
As a final remark, we observe that the above approach is rather brittle, since $\kappa_W$ is unknown (because $W$ is),  and if the condition $\kappa_W \le 1+\gamma$ fails, then once again the binary search can return a clustering far from the correct one.


\section{Comparison with metric learning}
\label{app:metric}
In this section we show that metric learning, a common approach to latent cluster recovery and related problems, does not solve our problem even when combined with same-cluster and comparison queries.
Intuitively, we want to learn an approximate distance $\hat{d}$ that preserves the ordering of the distances between the points.
That is, for all $\bx,\by,\bz \in X$, $d(\bx,\by) \le d(\bx,\bz)$ implies $\hat{d}(\bx,\by) \le \hat{d}(\bx,\bz)$.
If this holds then $d$ and $\hat{d}$ are equivalent from the point of view of binary search.
To simplify the task, we may equip the algorithm with an additional \emph{comparison query} \cmp, which takes in input two pairs of points $\bx,\bx'$ and $\by,\by'$ from $X$ and tells precisely whether $d(\bx,\bx') \le d(\by,\by')$ or not. It turns out that, even with \scq+\cmp\ queries, learning such a $\hat{d}$ requires to query essentially all the input points.
\begin{theorem}
\label{thm:dir}
For any $d \ge 3$, learning any $\hat{d}$ such that, for all $\bx,\by,\bz \in X$, if $d(\bx,\by) \le d(\bx,\bz)$ then $\hat{d}(\bx,\by) \le \hat{d}(\bx,\bz)$, requires $\Omega(n)$ \scq+\cmp\ queries in the worst case, even with an arbitrarily large margin $\gamma$.
\end{theorem}
\begin{proof}
We reduce the problem of learning the  order of pairwise distances induced by $W$, which we call ORD, to the problem of learning a separator hyperplane, which we call SEP and whose query complexity is linear in $n$.

Problem SEP is as follows.
The inputs are a set $X=\{\bx_1,\ldots,\bx_n\} \subset \mathbb{R}^d$ (the observations) and a set $\mathcal{H} = \{\bh_1,\ldots,\bh_k\} \subset \mathbb{R}_+^d$ (the hypotheses).
We require that $\bh_j \in \R^d_+$.
We have oracle access to $\sigma : X \to \{+1,-1\}$ such that $\sigma(\cdot) = \sign\langle{\bh},\cdot\rangle$ for some $\bh \in \mathcal{H}$.
The output is the $\bh \in \mathcal{H}$ that agrees with $\sigma$.
We assume $\mathcal{H},X$ support a margin: $\exists \epsilon>0$, possibly dependent on the instance, such that $\sign\langle{\bh},\bx\rangle = \sign\langle{\bh},\bx'\rangle$ for all $\bx'$ with $\|\bx - \bx'\| \le \epsilon$.
(Note that this is \emph{not} the cluster margin $\gamma$).

Let $Q_{\text{ORD}}(n)$ and $Q_{\text{SEP}}(n)$ be the query complexities of ORD and SEP on $n$ points.
We show:
\begin{lemma}
$Q_{\text{ORD}}(3n) \le Q_{\text{SEP}}(n)$.
\label{lem:DIR}
\end{lemma}
\begin{proof}
Let $X=\{\bx_1,\ldots,\bx_n\} \subseteq \R^d$ be the input points for SEP and let $\bh \in \R^d_+$ be the target hypothesis.
By scaling the dataset we can assume $\|\bx_i\| \le \epsilon$ for any desired $\epsilon$ (even dependent on $n$).
We define an instance of ORD on $n'=3n$ points as follows.
First, $W = \diag(\bh)$.
Second, the input set is $X'=S_1 \cup \ldots \cup S_n$ where for $i=1,\ldots,n$ we define $S_i=\{\ai, \bi, \ci\}$ with:
\begin{align}
\ai &= 6^i \cdot \mathbf{1} \\
\bi &= 2 \cdot \ai \\
\ci &= 3 \cdot \ai + \bx_i
\end{align}
%
We first show that a solution to ORD gives a solution of SEP.
Suppose indeed that for all pairs of points $\{\bq,\bp\},\{\bx,\by\}$ we know whether $d_{W}(\bq,\bp) \le d_W(\bx,\by)$.
This is equivalent to knowing the output of $\cmp(\{\bq,\bp\},\{\bx,\by\})$, which is
\begin{align}
\cmp(\{\bq,\bp\},\{\bx,\by\}) &=
\sign \dotp{\bh, (\bq-\bp)^2 - (\bx-\by)^2} 
\end{align}
Consider then the point $\bq=\ci, \bp=\bx=\bi, \by=\ai$ for each $i$.
Then:
\begin{align}
\cmp(\{\bq,\bp\},\{\bx,\by\}) &=
\sign \dotp{\bh, (\ai-\bi)^2 - (\bi-\ci)^2} 
\\ &=
\sign \dotp{\bh, (\ai)^2 - (-\ai-\bx_i)^2} \\
&=
\sign \dotp{\bh, 2 \cdot 6^i \bx_i - \bx_i^2}
\\&=
\sign \dotp{\bh, \bx_i \left( 1 - \frac{\bx_i}{2 \cdot 6^i} \right)}
\end{align}
By the margin hypothesis, for $\epsilon$ small enough this equals $\sign( \dotp{\bh, \bx_i})$, i.e., the label of $\bx_i$ in SEP.

We now show that all the other queries reveal no information about the solution of SEP.
Suppose then the points are not in the form $\bq=\ci, \bp=\bx=\bi, \by=\ai$.
Without loss of generality, we can assume that $\bq>\bp$ and $\bq \ge \bx>\by$.
It is then easy to see that, for $\epsilon$ small enough, $(\bq-\bp)^2-(\bx-\by)^2 > 0$
or $(\bq-\bp)^2-(\bx-\by)^2 < 0$.
This holds independently of the $\bx_i$ and of $W$ and therefore gives no information about the solution of SEP.

It follows that, if we can solve ORD in $f(3n)$ \cmp\ queries, then we can solve SEP in $f(n)$ queries.
Finally, note that adding \scq\ queries does not reduce the query complexity (e.g., let $X$ lie in a single cluster).
For the same reason, we can even assume an arbitrarily large cluster margin $\gamma$.
\end{proof}
It remains to show that SEP requires $\Omega(n)$ \cmp\ queries in the worst case.
This is well known, but we need to ensure that $\mathcal{H} \subset \R^d_+$ and that any $h \in \mathcal{H}$ supports a margin as described above.

Consider the following set $X = \{\bx_1,\ldots,\bx_n\} \subseteq \mathbb{R}^3$:
\begin{align}
\bx_i = (1-\delta, -\cos(\theta_i), -\sin(\theta_i))
\end{align}
where $\theta_i = i \frac{\pi}{2n}$ and $\delta$ is sufficiently small.
Let $\mathcal{H}= \{\bh_1,\ldots,\bh_n\}$, where
\begin{align}
\bh_j = (1,\cos(\theta_j),\sin(\theta_j))
\end{align}
Note that $\mathcal{H} \subset \R^d_+$ as required.
Clearly:
\begin{align}
\dotp{\bh_j, \bx_i} = \left\{ \begin{array}{ll}
-\delta & \text{if } j=i \\
1-(\delta + \cos(\theta_i-\theta_j)) & \text{if } j \ne i
\end{array} \right.
\end{align}
By choosing $\delta = \frac{1-\cos(\pi/2n)}{2}$ we have $\sign \dotp{\bh, \bx_i} = -1$ if and only if $i=j$.
Clearly, any algorithm needs to probe $\Omega(n)$ labels to learn $h$ with constant probability for some $h \in \mathcal{H}$.
Finally, note that any $h$ supports a margin, as required.
\end{proof}

\end{document}